\relax
\documentclass[a4paper,11pt]{article}
\usepackage[square, sort]{natbib}
\usepackage[nottoc,notlot,notlof]{tocbibind}
\usepackage{times}  
\usepackage{helvet} 
\usepackage{courier}  
\usepackage{url}
\usepackage{fullpage}

\usepackage{xcolor}
\definecolor{darkblue}{rgb}{0,0.08,0.8}

\usepackage{subcaption}
\usepackage{lmodern}
\usepackage{datatool}
\usepackage{algorithm}
\usepackage[noend]{algpseudocode}

\usepackage{amsthm,amsmath,amssymb}
\usepackage{bm}
\usepackage{latexsym}
\usepackage{float}
\usepackage{color}
\usepackage{pifont}
\newtheorem{theorem}{Theorem}[section]
\newtheorem*{theorem*}{Theorem}

\newtheorem{proposition}[theorem]{Proposition}
\newtheorem*{proposition*}{Proposition}
\newtheorem{definition}[theorem]{Definition}

\renewcommand{\hat}{\widehat}
\renewcommand{\tilde}{\widetilde}
\newcommand{\argmax}{\mathop{\rm arg~max}\limits}
\newcommand{\argmin}{\mathop{\rm arg~min}\limits}
\newcommand{\A}{\mathrm{a}}
\newcommand{\C}{\mathrm{c}}
\newcommand{\U}{\mathrm{u}}
\newcommand{\E}{\mathbb{E}}
\newcommand{\R}{\mathbb{R}}
\newcommand{\D}{\mathcal{D}}
\usepackage[colorlinks=true,linkcolor=darkblue, citecolor=darkblue]{hyperref}

\usepackage{booktabs}       
\usepackage{nicefrac}       
\usepackage{microtype}      
\usepackage{multirow}
\usepackage{siunitx}

\usepackage{microtype}
\usepackage{mathtools}
\usepackage{adjustbox}
\usepackage{graphicx}
\usepackage[english]{babel}
\usepackage[utf8]{inputenc}
\usepackage{caption}

\newcolumntype{C}{>{$}c<{$}}
\newcolumntype{L}{>{$}l<{$}}

\begin{document}
\title{Imitation Learning from Imperfect Demonstration}
\author{Yueh-Hua Wu$^{1,2}$, Nontawat Charoenphakdee$^{2,3}$, Han Bao$^{2,3}$, \\Voot Tangkaratt$^{2}$, Masashi Sugiyama$^{3,2}$}
\date{
$^1$ National Taiwan University\\ 
\vspace{0.03in}
$^2$ RIKEN Center for Advanced Intelligence Project\\
\vspace{0.05in}
$^3$ The University of Tokyo
}
\maketitle

\begin{abstract}
Imitation learning (IL) aims to learn an optimal policy from demonstrations.
However, such demonstrations are often imperfect since collecting optimal ones is costly.
To effectively learn from imperfect demonstrations, we propose a novel approach that utilizes \emph{confidence scores}, which describe the quality of demonstrations.
More specifically, we propose two confidence-based IL methods, namely two-step importance weighting IL~(2IWIL) and generative adversarial IL with imperfect demonstration and confidence~(IC-GAIL). 
We show that confidence scores given only to a small portion of sub-optimal demonstrations significantly improve the performance of IL both theoretically and empirically.
 
\end{abstract}

\section{Introduction}
Imitation learning (IL) has become of great interest because obtaining demonstrations is usually easier than designing reward. 
Reward is a signal to instruct agents to complete the desired tasks. 
However, ill-designed reward functions usually lead to unexpected behaviors~\citep{everitt2016avoiding,dewey2014reinforcement,amodei2016concrete}. 
There are two main approaches that can be used to solve IL: behavioral cloning (BC)~\citep{schaal1999imitation}, which adopts supervised learning approaches to learn an action predictor that is trained directly from demonstration data; and apprenticeship learning (AL), which attempts to find a policy that is better than the expert policy for a class of cost functions~\citep{abbeel2004apprenticeship}. 
Even though BC can be trained with supervised learning approaches directly, it has been shown that BC cannot imitate the expert policy without a large amount of demonstration data for not considering the transition of environments~\citep{ross2011reduction}. 
In contrast, AL approaches learn from interacting with environments and optimize objectives such as maximum entropy~\citep{ziebart2008maximum}.

A state-of-the-art approach \emph{generative adversarial imitation learning} (GAIL) is proposed by \citet{ho2016generative}. 
The method learns an optimal policy by performing occupancy measure matching~\citep{syed2008apprenticeship}. 
An advantage of the matching method is that it is robust to demonstrations generated from a stochastic policy.
Based on the concept proposed in GAIL, variants have been developed recently for different problem settings~\citep{li2017infogail,kostrikov2018addressing}. 

Despite that GAIL is able to learn an optimal policy from optimal demonstrations, to apply IL approaches to solve real-world tasks, the difficulty in obtaining such demonstration data should be taken into consideration. 
However, demonstrations from an optimal policy (either deterministic or stochastic) are usually assumed to be available in the above mentioned works, which can be barely fulfilled by the fact that most of the accessible demonstrations are imperfect or even from different policies. 
For instance, to train an agent to play basketball with game-play videos of the National Basketball Association, we should be aware that there are 14.3 turnovers per game\footnote{\url{https://www.basketball-reference.com/leagues/NBA_stats.html}}, not to mention other kinds of mistakes that may not be recorded. 
The reason why optimal demonstrations are hard to obtain can be attributed to the limited attention and the presence of distractions, which make humans hard to follow optimal policies all the time. As a result, some parts of the demonstrations may be optimal and the others are not.

To mitigate the above problem, we propose to use confidence scores, which indicate the probability that whether a given trajectory is optimal.  
In practice, obtaining confidence scores can be cheaper than collecting optimal demonstrations.
It is because it requires merely the knowledge of the optimal behavior to score but performing optimally requires not only such knowledge but also strict physical conditions. For instance, to play basketball well, the capabilities of making spontaneous decisions and intrinsic fingertip control are required. 
Therefore, for real-world tasks, the confidence labelers are not necessarily expert at achieving the goal. They can be normal enthusiasts such as audiences of basketball games. 

To further reduce the additional cost to learn an optimal policy, we consider a more realistic setting that the given demonstrations are partially equipped with confidence. As a result, the goal of this work is to utilize imperfect demonstrations where some are equipped with confidence while some are not (we refer to demonstrations without confidence as ``unlabeled demonstrations'').

In this work, we consider the setting where the given imperfect demonstrations are a mixture of optimal and non-optimal demonstrations. 
The setting is common when the demonstrations are collected via crowdsourcing \citep{serban2017deep,hu2018inference,shah2018bootstrapping} and learning from different sources such as videos \citep{tokmakov2017learning,pathak2017learning,supancic2017tracking,yeung2017learning,liu2018imitation}, where demonstrations can be generated from different policies. 

We propose two methods, \emph{two-step importance weighting imitation learning} (2IWIL) and \emph{generative adversarial imitation learning with imperfect demonstration and confidence} (IC-GAIL), based on the idea of reweighting but from different perspectives. 
To utilize both confidence and unlabeled data, for 2IWIL, it predicts confidence scores for unlabeled data by optimizing the proposed objective based on empirical risk minimization (ERM)~\citep{vapnik1998statistical}, which has flexibility for different loss functions, models, and optimizers; on the other hand, instead of directly reweighting to the optimal distribution and perform GAIL with reweighting, IC-GAIL reweights to the \emph{non-optimal} distribution and match the optimal occupancy measure based on our mixture distribution setting.
Since the derived objective of IC-GAIL depends on the proportion of the optimal demonstration in the demonstration mixture, we empirically show that IC-GAIL converges slower than 2IWIL but achieves better performance, which forms a trade-off between the two methods. We show that the proposed methods are both theoretically and practically sound.

\section{Related work}\label{relatedwork}
In this section, we provide a brief survey about making use of non-optimal demonstrations and semi-supervised classification with confidence data.

\subsection{Learning from non-optimal demonstrations}
Learning from non-optimal demonstrations is nothing new in IL and reinforcement learning (RL) literature, but previous works utilized different information to learn a better policy. \emph{Distance minimization inverse RL} (DM-IRL) \citep{burchfiel2016distance} utilized a feature function of states and assumed that the true reward function is linear in the features. The feedback from human is an estimate of accumulated reward, which is harder to be given than confidence because multiple reward functions may correspond to the same optimal policy.

\emph{Semi-supervised IRL} (SSIRL)~\citep{valko2012semi} extends the IRL method proposed by \citet{abbeel2004apprenticeship}, where the reward function can be learned by matching the \emph{feature expectations} of the optimal demonstrations. The difference from \citet{abbeel2004apprenticeship} is that in SSIRL, optimal and sub-optimal trajectories from other performers are given. Transductive SVM~\citep{soentpiet1999advances} was used in place of vanilla SVM in \citet{abbeel2004apprenticeship} to recognize optimal trajectories in the sub-optimal ones. In our setting, the confidence scores are given instead of the optimal demonstrations. DM-IRL and SSIRL are not suitable for high-dimensional problems due to its dependence on the linearity of reward functions and good feature engineering. 

\subsection{Semi-supervised classification with confidence data}
In our 2IWIL method, we train a probabilistic classifier with confidence and unlabeled data by optimizing the proposed ERM objective. There are similar settings such as \emph{semi-supervised classification} \citep{chapelle2009semi}, where few hard-labeled data $y\in\{0,1\}$ and some unlabeled data are given.

\citet{zhou2014soft} proposed to use hard-labeled instances to estimate confidence scores for unlabeled samples using Gaussian mixture models and principal component analysis. Similarly, for an input instance $x$, \citet{wang2013weakly} obtained an upper bound of confidence $\Pr(y=1\vert x)$ with hard-labeled instances and a kernel density estimator, then treated the upper bound as an estimate of probabilistic class labels. 

Another related scheme was considered in \citet{el2010semi} where they considered soft labels $z\in[0,1]$ as fuzzy inputs and proposed a classification approach based on k-nearest neighbors. This method is difficult to scale to high-dimensional tasks, and lacks theoretical guarantees.
\citet{ishida2018binary} proposed another scheme that trains a classifier only from positive data equipped with confidence.
Our proposed method, 2IWIL, also considers training a classifier with confidence scores of given demonstrations.
Nevertheless, 2IWIL can train a classifier from fewer confidence data, with the aid of a large number of unlabeled data.

\section{Background}\label{background}
In this section, we provide backgrounds of RL and GAIL.

\subsection{Reinforcement Learning}
We consider the standard Markov Decision Process (MDP) \citep{sutton1998introduction}. MDP is represented by a tuple $\langle S,\mathcal{A},\mathcal{P}, \mathcal{R},\gamma\rangle$, where $S$ is the state space, $\mathcal{A}$ is the action space, $\mathcal{P}(s_{t+1}\vert s_t,a_t)$ is the transition density of state $s_{t+1}$ at time step $t+1$ given action $a_t$ made under state $s_t$ at time step $t$, $\mathcal{R}(s,a)$ is the reward function, and $\gamma\in(0,1)$ is the discount factor.

A stochastic policy $\pi(a\vert s)$ is a density of action~$a$ given state~$s$. The performance of $\pi$ is evaluated in the $\gamma$-discounted infinite horizon setting and its expectation can be represented with respect to the trajectories generated by~$\pi$:
\begin{align}\label{eq:reward_sum}
    \mathbb{E}_\pi[\mathcal{R}(s,a)]=\mathbb{E}\left[\sum_{t=0}^\infty \gamma^t\mathcal{R}(s_t,a_t)\right],
\end{align}
where the expectation on the right-hand side is taken over the densities $p_0(s_0)$, $\mathcal{P}(s_{t+1}|s_t,a_t)$, and $\pi(a_t|s_t)$ for all time steps $t$.
Reinforcement learning algorithms~\citep{sutton1998introduction} aim to maximize Eq.~\eqref{eq:reward_sum} with respect to $\pi$.

To characterize the distribution of state-action pairs generated by an arbitrary policy $\pi$, the occupancy measure is defined as follows.
\begin{definition}[\citet{Puterman:1994:MDP:528623}]
Define occupancy measure $\rho_\pi:S\times \mathcal{A}\rightarrow \mathbb{R}$,
\begin{align}
    \rho_\pi(s,a)=\pi(a\vert s)\sum_{t=0}^\infty \gamma^t\Pr(s_t=s\vert \pi),
\end{align}
where $\Pr(s_t=s\vert \pi)$ is the probability density of state $s$ at time step $t$ following policy $\pi$.
\end{definition}
The occupancy measure of $\pi$, $\rho_\pi(s,a)$, can be interpreted as an unnormalized density of state-action pairs. The occupancy measure plays an important role in IL literature because of the following one-to-one correspondence with the policy.
\begin{theorem}\label{theorem:onetoone}
(Theorem 2 of \citet{syed2008apprenticeship}) Suppose $\rho$ is the occupancy measure for $\pi_{\rho}(a\vert s)\triangleq\frac{\rho(s,a)}{\sum_{a'}\rho(s,a')}$. Then $\pi_\rho$ is the only policy whose occupancy measure is $\rho$.
\end{theorem}

In this work, we also define the \emph{normalized occupancy measure} $p(s,a)$,
\begin{align*}
    p(s,a)\triangleq&\frac{\rho(s,a)}{\sum_{s,a}\rho(s,a)}\\
    =&\frac{\rho(s,a)}{\sum_{s,a}\pi(a\vert s)\sum_{t=0}^\infty \gamma^t\Pr(s_t=s\vert \pi)}\\
    =&\frac{\rho(s,a)}{\sum_{t=0}^\infty \gamma^t}=(1-\gamma)\rho(s,a).
\end{align*}
The normalized occupancy measure can be interpreted as a probability density of state-action pairs that an agent experiences in the environment with policy $\pi$. 

\subsection{Generative adversarial imitation learning (GAIL)}
The problem setting of IL is that given trajectories $\{(s_i,a_i)\}_{i=1}^n$ generated by an expert $\pi_\mathrm{E}$, we are interested in optimizing the agent policy $\pi_\theta$ to recover the expert policy $\pi_\mathrm{E}$ with $\{(s_i,a_i)\}_{i=1}^n$ and the MDP tuple without reward function $\mathcal{R}$. 

GAIL~\citep{ho2016generative} is a state-of-the-art IL method that performs occupancy measure matching to learn a parameterized policy. Occupancy measure matching aims to minimize the objective $d(\rho_{\pi_\mathrm{E}},\rho_{\pi_\theta})$, where $d$ is a distance function. The key idea behind GAIL is that it uses generative adversarial training to estimate the distance and minimize it alternatively. To be precise, the distance is the Jensen-Shannon divergence (JSD), which is estimated by solving a binary classification problem. This leads to the following min-max optimization problem:
\begin{align}\label{eq:gail}
    \min_\theta\max_w
    \mathbb{E}_{s,a\sim p_\theta}[\log D_w(s,a)]+\mathbb{E}_{s,a\sim p_\mathrm{opt}}[\log (1-D_w(s,a))],
\end{align}
where $p_\theta$ and $p_\mathrm{opt}$ are the corresponding normalized occupancy measures for $\pi_\theta$ and $\pi_\mathrm{opt}$ respectively. $D_w$ is called a discriminator and it can be shown that if the discriminator has infinite capacity, the global optimum of Eq.~\eqref{eq:gail} corresponds to the JSD up to a constant \citep{goodfellow2014generative}. To update the agent policy $\pi_\theta$, GAIL treats the loss $-\log(D_w(s,a))$ as a reward signal and the agent can be updated with RL methods such as trust region policy optimization (TRPO)~\citep{schulman2015trust}. A weakness of GAIL is that if the given demonstrations are non-optimal then the learned policy will be non-optimal as well.

\section{Imitation learning with confidence and unlabeled data}\label{methods}
In this section, we present two approaches to learning from imperfect demonstrations with confidence and unlabeled data. 
The first approach is \emph{2IWIL}, which aims to learn a probabilistic classifier to predict confidence scores of unlabeled demonstration data and then performs standard GAIL with reweighted distribution.
The second approach is \emph{IC-GAIL}, which forgoes learning a classifier and learns an optimal policy by performing occupancy measure matching with unlabeled demonstration data.
Details of derivation and proofs in this section can be found in Appendix.

\subsection{Problem setting}
Firstly, we formalize the problem setting considered in this paper.
For conciseness, in what follows we use $x$ in place of $(s,a)$.
Consider the case where given imperfect demonstrations are sampled an optimal policy $\pi_\mathrm{opt}$ and non-optimal policies $\Pi=\{\pi_i\}_{i=1}^n$.
Denote that the corresponding normalized occupancy measure of $\pi_\mathrm{opt}$ and $\Pi$ are $p_\mathrm{opt}$ and $\{p_i\}_{i=1}^n$, respectively.
The normalized occupancy measure $p(x)$ of a state-action pair $x$ is therefore the weighted sum of $p_\mathrm{opt}$ and $\{p_i\}_{i=1}^n$,
\begin{align*}
    p(x)=&\alpha p_\mathrm{opt}(x)+\sum_{i=1}^{n} \nu_ip_i(x)
    \\=&\alpha p_\mathrm{opt}(x)+(1-\alpha)p_\mathrm{non}(x),
\end{align*}
where $\alpha + \sum_{i=1}^n\nu_i=1$ and $p_\mathrm{non}(x)=\frac{1}{(1-\alpha)}\sum_{i=1}^n\nu_ip_i(x)$.
We may further follow traditional classification notation by defining $p_\mathrm{opt}(x)\triangleq p(x\vert y=+1)$ and $p_\mathrm{non}(x)\triangleq p(x\vert y=-1)$, where $y=+1$ indicates that $x$ is drawn from the occupancy measure of the optimal policy and $y=-1$ indicates the non-optimal policies.
Here, $\alpha = \Pr(y=+1)$ is the class-prior probability of the optimal policy.
We further assume that an oracle labels state-action pairs in the demonstration data with \emph{confidence scores} $r(x)\triangleq p(y=+1|x)$. Based on this, the normalized occupancy measure of the optimal policy can be expressed by the Bayes' rule as
\begin{align}\label{eq:bayes}
    p(x\vert y=+1)=&\frac{r(x)p(x)}{\alpha}.
\end{align}
We assume that labeling state-action pairs by the oracle can be costly and only some pairs are labeled with confidence. More precisely, we obtain demonstration datasets as follows,
\begin{align*}
    &\mathcal{D}_\C\triangleq\{(x_{\C,i},r_i)\}_{i=1}^{n_\C}\overset{\mathrm{i.i.d.}}{\sim} q(x,r),\\
    &\mathcal{D}_\U\triangleq\{x_{\U,i}\}_{i=1}^{n_\U}\overset{\mathrm{i.i.d.}}{\sim} p(x),
\end{align*}
where $q(x,r)=p(x)p_\mathrm{r}(r\vert x)$ and $p_\mathrm{r}(r_i\vert x)=\delta(r_i-r(x))$ is a delta distribution.
Our goal is to consider the case where $\mathcal{D}_c$ is scarce and we want to learn the optimal policy $\pi_\mathrm{opt}$ with $\mathcal{D}_c$ and $\mathcal{D}_u$ jointly.

\subsection{Two-step importance weighting imitation learning}
We first propose an approach based on the importance sampling scheme.
By Eq.~\eqref{eq:bayes}, the GAIL objective in Eq.~\eqref{eq:gail} can be rewritten as follows:
\begin{align}\label{eq:gail_reweight}
    \min_\theta \max_w\;&
    \mathbb{E}_{x\sim p_\theta}\left[\log D_w(x)\right]+\mathbb{E}_{x,r\sim q}\left[\frac{r}{\alpha}\log(1-D_w(x))\right].
\end{align}
In practice, we may use the mean of confidence scores to estimate the class prior $\alpha$. Although we can reweight the confidence data $\mathcal{D}_\C$ to match the optimal distribution, 
we have a limited number of confidence data and it is difficult to perform accurate sample estimation.
To make full use of unlabeled data, the key idea is to identify confidence scores of the given unlabeled data $\mathcal{D}_\U$ and reweight both confidence data and unlabeled data.
To achieve this, we train a probabilistic classifier from confidence data and unlabeled data, where we call this learning problem \emph{semi-conf (SC) classification}. 




Let us first consider a standard binary classification problem to classify samples into $p_\mathrm{opt}$ ($y=+1$) and $p_\mathrm{non}$ ($y=-1$).
Let $g\colon\R^d \to \R$ be a prediction function and $\ell\colon \R \to \R_+$ be a loss function. The optimal classifier can be learned by minimizing the following risk:
\begin{align}\label{eq:pnrisk}
R_{\mathrm{PN},\ell}(g) &= \alpha \mathbb{E}_{x\sim p_\mathrm{opt}} \left[ \ell(g(x)) \right]+ (1-\alpha) \mathbb{E}_{x\sim p_\mathrm{non}} \left[ \ell(-g(x)) \right] \text{,}
\end{align}
where PN stands for ``positive-negative''.
However, as we only have samples from the mixture distribution $p$ instead of samples separately drawn from $p_\mathrm{opt}$ and $p_\mathrm{non}$, it is not straightforward to conduct sample estimation of the risk in Eq.~\eqref{eq:pnrisk}.
To overcome this issue, we express the risk in an alternative way that can be estimated only from $\mathcal{D}_\C$ and $\mathcal{D}_\U$ in the following theorem.
\begin{theorem}
\label{lemma:unlabel}
The classification risk~\eqref{eq:pnrisk} can be equivalently expressed as 
\begin{align}\label{eq:weighted}
    R_{\mathrm{SC},\ell}(g)&= \mathbb{E}_{x,r\sim q}[r(\ell(g(x))-\ell(-g(x)))+(1-\beta)\ell(-g(x))]+\mathbb{E}_{x\sim p}[\beta \ell(-g(x))],
\end{align}
where $\beta\in[0,1]$ is an arbitrary weight.
\end{theorem}


Thus, we can obtain a probabilistic classifier by minimizing Eq.~\eqref{eq:weighted},
which can be estimated only with $\mathcal{D}_\C$ and $\mathcal{D}_\U$.
Once we obtain the prediction function $g$,
we can use it to give confidence scores for $\mathcal{D}_\U$.

To make the prediction function $g$ estimate confidence accurately, the loss function $\ell$ in Eq.~\eqref{eq:weighted} should come from a class of \emph{strictly proper composite loss}~\citep{buja2005loss,reid2010composite}.
Many losses such as the squared loss, logistic loss, and exponential loss are proper composite. 
For example, if we obtain $g_{\mathrm{log}}^*$ that minimizes a logistic loss $\ell_{\log}(z)=(\log(1+\exp(-z))$, we can obtain confidence scores by passing prediction outputs to a sigmoid function $\hat{p}(y=1|x)=[1+\exp(-g_{\log}^*(x))]^{-1}$~\citep{reid2010composite}.
On the other hand, the hinge loss cannot be applied since it is not a proper composite loss and cannot estimate confidence reliably~\citep{bartlett2007sparseness,reid2010composite}. 
Therefore, we can obtain a probabilistic classifier from the prediction function $g$ that learned from a strictly proper composite loss.
After obtaining a probabilistic classifier, we optimize the importance weighted objective in Eq.~\eqref{eq:gail_reweight},
where both $\mathcal{D}_\C$ and $\mathcal{D}_\U$ are used to estimate the second expectation.
We summarize this training procedure in Algorithm~\ref{alg:2iwil}.

Next, we discuss the choice of the combination coefficient~$\beta$.
Since we have access to the empirical unbiased estimator $\Hat{R}_{\mathrm{SC},\ell}(g)$ from Eq.~\eqref{eq:weighted}, it is natural to find the minimum variance estimator among them.
The following theorem gives the optimal $\beta$ in terms of the estimator variance.
\begin{proposition}[variance minimality]
\label{proposition:min_var}
Let $\sigma_\mathrm{cov}$ denote the covariance between $n_\C^{-1}\sum_{i=1}^{n_\C}r_i\{\ell(g(x_{\C,i}))-\ell(-g(x_{\C,i}))\}$ and $n_\C^{-1}\sum_{i=1}^{n_\C}\ell(-g(x_{\C,i}))$.
For a fixed $g$, the estimator $\hat{R}_{\mathrm{SC},\ell}(g)$
has the minimum variance when $\beta=\mathrm{clip}_{[0, 1]}(\frac{n_\U}{n_\C+n_\U}+\frac{\sigma_\mathrm{cov}}{\mathrm{Var}(\ell(-g(x)))}\frac{n_\C n_\U}{n_\C+n_\U})$.\footnote{$\mathrm{clip}_{[l, u]}(v) \triangleq \max\{l, \min\{v, u\}\}$.}
\end{proposition}
Thus, $\beta$ lies in $(0, 1)$ when the covariance $\sigma_{\mathrm{cov}}$ is not so large.
If $\beta \ne 0$, it means that the unlabeled data $\mathcal{D}_\U$ does help the classifier by reducing empirical variance when Eq.~\eqref{eq:weighted} is adopted.
However, computing the $\beta$ that minimizes empirical variance is computationally inefficient since it involves computing $\sigma_\mathrm{cov}$ and $\mathrm{Var}(l(-g(x)))$.
In practice, we use $\beta = \frac{n_\U}{n_\C + n_\U}$ for all experiments by assuming that the covariance is small enough.

In our preliminary experiments, we sometimes observed that the empirical estimate $\hat{R}_{\mathrm{SC},\ell}$ of Eq.~\eqref{eq:weighted} became negative and led to overfitting.
We can mitigate this phenomenon by employing a simple yet highly effective technique from \citet{kiryo2017positive},
which is proposed to solve a similar overfitting problem (see Appendix for implementation details).
\begin{algorithm}[tb]
  \caption{2IWIL}
  \label{alg:2iwil}
\begin{algorithmic}[1]
  \State {\bfseries Input:} Expert trajectories and confidence $\mathcal{D}_\C = \{(x_{\C,i}, r_i)\}_{i=1}^{n_\C}$, $\mathcal{D}_\U = \{x_{\U,i}\}_{i=1}^{n_\U}$
  \State Estimate the class prior by $\hat{\alpha}=\frac{1}{n_\C}\sum_{i=1}^{n_\C}r_i$
  \State Train a probabilistic classifier by minimizing Eq.~\eqref{eq:weighted} with $\beta=\frac{n_u}{n_u+n_c}$
  \State Predict confidence scores $\{\Hat{r}_{\U,i}\}_{i=1}^{n_\U}$ for $\{x_{\U,i}\}_{i=1}^{n_\U}$
  \For{$i=0,1,2,...$}
  \State Sample trajectories $\{x_i\}_{i=1}^{n_a}\sim \pi_{\theta}$
  \State Update the discriminator parameters by maximizing Eq.~\eqref{eq:gail_reweight}
  \State Update $\pi_{\theta}$ with reward $-\log D_w(x)$ using TRPO
  \EndFor
  \State \textbf{end for}
\end{algorithmic}
\end{algorithm}

\subsubsection{Theoretical Analysis}

Below, we show that the estimation error of Eq.~\eqref{eq:weighted} can be bounded.
This means that its minimizer is asymptotically equivalent to the minimizer of the standard classification risk $R_{\mathrm{PN},\ell}$,
which provides a consistent estimator of $p(y=+1|x)$.
We provide the estimation error bound with Rademacher complexity~\citep{bartlett2002rademacher}.
Denote $\mathfrak{R}_n(\mathcal{G})$ be the Rademacher complexity of the function class $\mathcal{G}$ with the sample size $n$.
\begin{theorem}
    \label{theorem:2iwil-bound}
    Let $\mathcal{G}$ be the hypothesis class we use.
    Assume that the loss function $\ell$ is $\rho_\ell$-Lipschitz continuous,
    and that there exists a constant $C_\ell > 0$ such that $\sup_{x \in \mathcal{X}, y \in \{\pm 1\}}|\ell(yg(x))| \leq C_\ell$ for any $g \in \mathcal{G}$.
    Let $\hat{g} \triangleq \argmin_{g \in \mathcal{G}} \hat{R}_{\mathrm{SC},\ell}(g)$ and $g^* \triangleq \argmin_{g \in \mathcal{G}} R_{\mathrm{SC},\ell}(g)$.
    For $\delta \in (0, 1)$, with probability at least $1 - \delta$ over repeated sampling of data for training $\hat{g}$,
    \begin{align*}
        R_{\mathrm{SC},\ell}&(\hat{g}) - R_{\mathrm{SC},\ell}(g^*)
        \leq 16\rho_\ell((3 - \beta)\mathfrak{R}_{n_\C}(\mathcal{G}) + \beta\mathfrak{R}_{n_\U}(\mathcal{G}))
        + 4C_\ell\sqrt{\frac{\log(8/\delta)}{2}}\left((3 - \beta)n_\C^{-\frac{1}{2}} + \beta n_\U^{-\frac{1}{2}}\right).
    \end{align*}
\end{theorem}
Thus, we may safely obtain a probabilistic classifier by minimizing $\hat{R}_{\mathrm{SC},\ell}$,
which gives a consistent estimator.


\subsection{IC-GAIL}
Since 2IWIL is a two-step approach by first gathering more confidence data and then conducting importance sampling, the error may accumulate over two steps and degrade the performance.
Therefore, we propose IC-GAIL that can be trained in an end-to-end fashion and perform occupancy measure matching with the optimal normalized occupancy measure $p_\mathrm{opt}$ directly.

Recall that $p = \alpha p_{\mathrm{opt}} + (1 - \alpha)p_{\mathrm{non}}$. Our key idea here is to minimize the divergence between $p$ and $p'$, where $p' = \alpha p_\theta + (1 - \alpha)p_{\mathrm{non}}$.
Intuitively, the divergence between $p_\theta$ and $p_\mathrm{opt}$ is minimized if that between $p$ and $p'$ is minimized.
For Jensen-Shannon divergence, this intuition can be justified in the following theorem.

\begin{theorem}
\label{theorem:gan}
Denote that
\begin{align*}
    V(\pi_\theta,D_w)=\mathbb{E}_{x\sim p}[\log (1-D_w(x))]+\mathbb{E}_{x\sim p'}[\log D_w(x)],
\end{align*}
and that $C(\pi_\theta)=\max_wV(\pi_\theta,D_w)$.
Then, $V(\pi_\theta, D_w)$ is maximized when $D_w=\frac{p'}{p+p'}(\triangleq D_{w^*})$,
and its maximum value is $C(\pi_\theta)=-\log 4+2\mathrm{JSD}(p\|p')$.
Thus, $C(\pi_\theta)$ is minimized if and only if $p_\theta=p_\mathrm{opt}$ almost everywhere.
\end{theorem}

Theorem~\ref{theorem:gan} implies that the optimal policy can be found by solving the following objective,
\begin{align}\label{eq:icgail_origin}
    \min_\theta \max_w \mathbb{E}_{x\sim p}[\log (1-D_w(x))]+\mathbb{E}_{x\sim p'}[\log D_w(x)].
\end{align}
The expectation in the first term can be approximated from $\mathcal{D}_\U$,
while the expectation in the second term is the weighted sum of the expectation over $p_\theta$ and $p_{\mathrm{non}}$.
Data $\mathcal{D}_\mathrm{a}=\{x_{\mathrm{a},i}\}_i^{n_\mathrm{a}}$ sampled from $p_\theta$ can be obtained by executing the current policy $\pi_\theta$.
However, we cannot directly obtain samples from $p_{\mathrm{non}}$ since it is unknown.
To overcome this issue, we establish the following theorem.

\begin{theorem}\label{theorem:transform}
$V(\pi_\theta, D_w)$ can be transformed to $\tilde{V}(\pi_\theta, D_w)$, which is defined as follows:
\begin{align}\label{eq:icgail}
    \tilde{V}&(\pi_\theta, D_w) = \mathbb{E}_{x\sim p}[\log (1-D_w(x))]+\alpha\mathbb{E}_{x\sim p_\theta}[\log D_w(x)]+\mathbb{E}_{x,r\sim q}[(1-r)\log D_w(x)].
\end{align}
\end{theorem}
We can approximate Eq.~\eqref{eq:icgail} given finite samples $\mathcal{D}_\C$, $\mathcal{D}_\U$, and $\mathcal{D}_\mathrm{a}$.
In practice, we perform alternative gradient descent with respect to $\theta$ and $w$ to solve this optimization problem.
Below, we show that the estimation error of $\tilde{V}$ can be bounded for a fixed agent policy $\pi_\theta$.

\subsubsection{Theoretical analysis}
In this subsection, we show that the estimation error of Eq.~\eqref{eq:icgail} can be bounded, given a fixed agent policy $\pi_\theta$.
Let $\hat{V}(\pi_\theta, D_w)$ be the empirical estimate of Eq.~\eqref{eq:icgail}.

\begin{theorem}
    \label{theorem:icgail-bound}
    Let $\mathcal{W}$ be a parameter space for training the discriminator and $D_\mathcal{W} \triangleq \{D_w \mid w \in \mathcal{W}\}$ be its hypothesis space.
    Assume that there exist a constant $C_L > 0$ such that
    $|\log D_w(x)| \leq C_L$ and $|\log(1 - D_w(x))| \leq C_L$ for any $x \in \mathcal{X}$ and $w \in \mathcal{W}$.
    Assume that both $\log D_w(x)$ and $\log (1 - D_w(x))$ for any $w \in \mathcal{W}$ have Lipschitz norms no more than $\rho_L > 0$.
    For a fixed agent policy $\pi_\theta$, let $\{x_{\A,i}\}_{i=1}^{n_a}$ be a sample generated from $\pi_\theta$,
    $D_{\hat{w}} \triangleq \argmax_{w \in \mathcal{W}}\hat{V}(\pi_\theta, D_w)$,
    and $D_{w^*} \triangleq \argmax_{w \in \mathcal{W}}V(\pi_\theta, D_w)$.
    Then, for $\delta \in (0, 1)$, the following holds with probability at least $1 - \delta$:
    \begin{align*}
        V(\pi_\theta, D_{w^*}) - V(\pi_\theta, D_{\hat{w}})
        \leq 16\rho_L(\mathfrak{R}_{n_\U}(D_\mathcal{W}) + \alpha\mathfrak{R}_{n_\A}(D_\mathcal{W}) + \mathfrak{R}_{n_\C}(D_\mathcal{W}))
         + 4C_L\sqrt{\frac{\log(6/\delta)}{2}}\left(n_\U^{-\frac{1}{2}} + \alpha n_\A^{-\frac{1}{2}} + n_\C^{-\frac{1}{2}}\right).
    \end{align*}
\end{theorem}

Theorem~\ref{theorem:icgail-bound} guarantees that the estimation of Eq.~\eqref{eq:icgail} provides a consistent maximizer with respect to the original objective in Eq.~\eqref{eq:icgail_origin} at each step of the discriminator training.

\subsubsection{Practical implementation of IC-GAIL}
Even though Eq.~\eqref{eq:icgail} is theoretically supported, when the class prior $\alpha$ is low, the influence of the agent become marginal in the discriminator training.
This issue can be mitigated by thresholding $\alpha$ in Eq.~\eqref{eq:icgail} as follows:
\begin{align}\label{eq:practical-icgail}
    \min_\theta\max_w\mathbb{E}_{x\sim p}[\log (1-D_w(x))]+\lambda\mathbb{E}_{x\sim p_\theta}[\log D_w(x)]+(1-\lambda)\mathbb{E}_{x,r\sim q}\left[\frac{(1-r)}{(1-\alpha)}\log D_w(x)\right],
\end{align}
where $\lambda=\max\{\tau, \alpha\}$ and $\tau\in(0,1]$.
The training procedure of IC-GAIL is summarized in Algorithm~\ref{alg:icgail}.
Note that Eq.~\eqref{eq:practical-icgail} returns to Eq.~\eqref{eq:gail} and learns an sub-optimal policy when $\tau=1$.

\begin{algorithm}[tb]
  \caption{IC-GAIL}
  \label{alg:icgail}
\begin{algorithmic}[1]
  \State {\bfseries Input:} Expert trajectories, confidence, and weight threshold $\{x_{\U,i}\}_{i=1}^{n_\U}$, $\{(x_{\C,i}, r_i)\}_{i=1}^{n_\C}$, $\tau$
  \State Estimate the class prior by $\hat{\alpha}=\frac{1}{n_\C}\sum_{i=1}^{n_\C}r_i$
  \State $\lambda=\max\{\tau, \hat{\alpha}\}$
  \For{$i=0,1,2,...$}
  \State Sample trajectories $\{x_i\}_{i=1}^{n_a}\sim \pi_{\theta}$
  \State Update the discriminator parameters by maximizing Eq.~\eqref{eq:practical-icgail}
  \State Update $\pi_{\theta}$ with reward $-\log D_w(x)$ using TRPO
  \EndFor
  \State \textbf{end for}
\end{algorithmic}
\end{algorithm}

\subsection{Discussion}
\label{sec:discussion}
To understand the difference between 2IWIL and IC-GAIL, we discuss it from three different perspectives: unlabeled data, confidence data, and the class prior. 

\textbf{Role of unlabeled data:}
It should be noted that unlabeled data plays different roles in the two methods. In 2IWIL, we show that unlabeled data reduces the variance of the empirical risk estimator as shown in Proposition~\ref{proposition:min_var}.

On the other hand, in addition to making more accurate estimation, the usefulness of unlabeled data in IC-GAIL is similar to guided exploration~\citep{kang2018policy}. 
We may analogize confidence information in the imperfect demonstration setting to reward functions since both of them allow agents to learn an optimal policy in IL and RL, respectively.
Likewise, fewer confidence data can be analogous to sparse reward functions.
Even though a small number of confidence data and sparse reward functions do not make objective such as Eqs.~\eqref{eq:reward_sum} and~\eqref{eq:gail_reweight} biased,
they cause practical issues such as a deficiency in information for exploration.
To mitigate the problem, we imitate from sub-optimal demonstrations and use confidence information to refine the learned policy, which is similar to \citet{kang2018policy} in the sense that they imitate a sub-optimal policy to guide RL algorithms in the sparse reward setting.

\textbf{Role of confidence data:}
Confidence data is utilized to train a classifier and to reweight $p_\mathrm{opt}$ in 2IWIL,
which causes the two-step training scheme and therefore the error is accumulated in the prediction phase and the occupancy measure matching phase.
Differently, IC-GAIL instead compensates the $p_\mathrm{non}$ portion in the given imperfect demonstrations by mimicking the composition of $p$.
The advantage of IC-GAIL over 2IWIL is that it avoids the prediction error by employing an end-to-end training scheme.

\textbf{Influence of the class-prior $\alpha$:}
The class prior in 2IWIL as shown in Eq.~\eqref{eq:gail_reweight} serves as a normalizing constant so that the weight $\frac{r(x)}{\alpha}$ for reweighting $p$ to $p_\mathrm{opt}$ has unit mean. Consequently, the class prior $\alpha$ does not affect the convergence of the agent policy. 
On the other hand, the term with respect to the agent $p_\theta$ is directly scaled by $\alpha$ in Eq.~\eqref{eq:icgail} of IC-GAIL. To comprehend the influence, we may expand the reward function from the discriminator $-\log D_w^*(x)=-\log \left(\left(\frac{\alpha}{(1-\alpha)}p_\theta+p_\mathrm{non}\right)/\left(\frac{\alpha}{(1-\alpha)}(p_\mathrm{opt}+p_\theta)+2p_\mathrm{non}\right)\right)$ and it shows that the agent term is scaled by $\frac{\alpha}{(1-\alpha)}$, which makes the reward function prone to be a constant when $\alpha$ is small. Therefore the agent learns slower than in 2IWIL, where the reward function is $-\log\left(p_\theta/(p_\theta+p_\mathrm{opt})\right)$.

\begin{figure*}[t]
    \centering
    \includegraphics[scale=0.52]{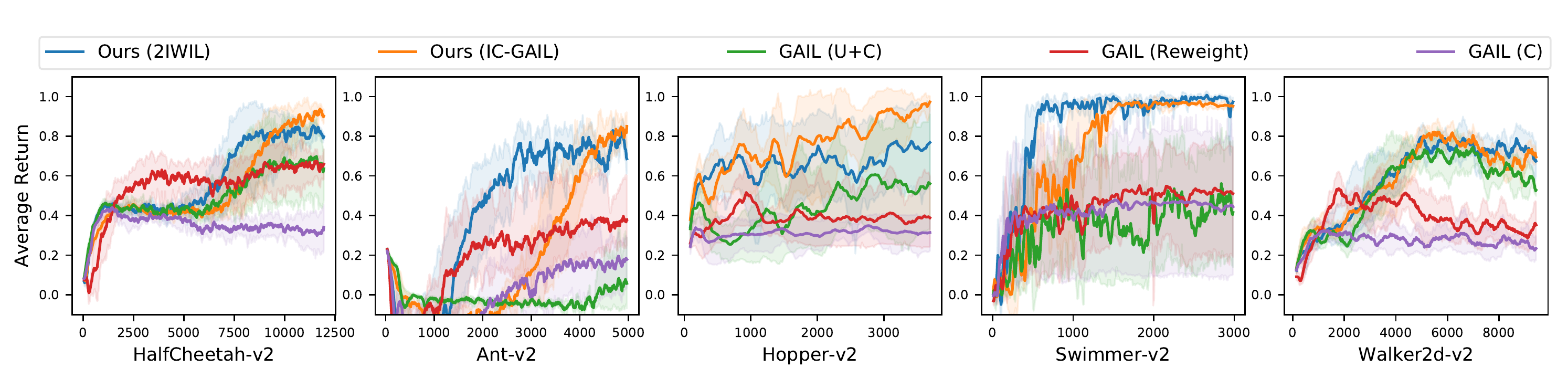}
    \caption{Learning curves of our 2IWIL and IC-GAIL versus baselines given imperfect demonstrations. The x-axis is the number of training iterations and the shaded area indicates standard error.}
    \label{fig:summary}
\end{figure*}

\section{Experiments}\label{experiments}
In this section, we aim to answer the following questions with experiments. (1) \emph{Do 2IWIL and IC-GAIL methods allow agents to learn near-optimal policies when limited confidence information is given?} (2) \emph{Are the methods robust enough when the given confidence is less accurate?} and (3) \emph{Do more unlabeled data results in better performance in terms of average return?} The discussions are given in Sec.~\ref{exp:performance}, \ref{exp:noise}, and \ref{exp:unlabel} respectively.

\textbf{Setup  } To collect demonstration data, we train an optimal policy ($\pi_\mathrm{opt}$) using TRPO~\citep{schulman2015trust} and select two intermediate policies ($\pi_{1}$ and $\pi_{2}$). 
The three policies are used to generate the same number of state-action pairs. 
In real-world tasks, the confidence should be given by human labelers. We simulate such labelers by using a probabilistic classifier $p^\star(y=+1|x)$ pre-trained with demonstration data and randomly choose $20\%$ of demonstration data to label confidence scores $r(x)=p^\star(y=+1|x)$.

We compare the proposed methods against three baselines. Denote that $\mathcal{D}_\C^x\triangleq\{x_{\C,i}\}_{i=1}^{n_\C}$, $\mathcal{D}_\C^r\triangleq\{r_i\}_{i=1}^{n_\C}$, and $\mathcal{D}_\U^x\triangleq \mathcal{D}_\U$. GAIL (U+C) takes all the pairs as input without considering confidence. 
To show if reweighting using Eq.~\eqref{eq:gail_reweight} makes difference, GAIL (C) and GAIL (Reweight) use the same state-action pairs $\D_\C^x$ but GAIL (Reweight) additionally utilizes reweighting with confidence information $\D_\C^r$. The baselines and the proposed methods are summarized in Table~\ref{table:methods}.

To assess our methods, we conduct experiments on Mujoco~\citep{todorov2012mujoco}. Each experiment is performed with five random seeds. The hyper-parameter $\tau$ of IC-GAIL is set to $0.7$ for all tasks. To show the performance with respect to the optimal policy that we try to imitate, the accumulative reward is normalized with that of the optimal policy and a uniform random policy so that $1.0$ indicates the optimal policy and $0.0$ the random one. Due to space limit, we defer implementation details, the performance of the optimal and the random policies, the specification of each task, and the uncropped figures of Ant-v2 to Appendix.
\begin{table}[t]
\caption{Comparison between proposed methods (IC-GAIL and 2IWIL) and baselines.}
\label{table:methods}
\vskip 0.15in
\begin{center}
\begin{small}
\begin{sc}
\begin{tabular}{lccr}
\toprule
Method & Input & objective \\
\midrule
IC-GAIL  & $\mathcal{D}_\U\cup \mathcal{D}_\C$ & Eq.~\eqref{eq:icgail}\\
2IWIL   & $\mathcal{D}_\U\cup \mathcal{D}_\C$ & Eq.~\eqref{eq:weighted}\\
\midrule
GAIL (U+C)  & $\mathcal{D}_\U\cup \mathcal{D}_\C^x$&Eq.~\eqref{eq:gail}\\
GAIL (C)  & $\mathcal{D}_\C^x$&Eq.~\eqref{eq:gail}\\
GAIL (reweight)  & $\mathcal{D}_\C$&Eq.~\eqref{eq:gail_reweight}\\
\bottomrule
\end{tabular}
\end{sc}
\end{small}
\end{center}
\end{table}
\subsection{Performance comparison}\label{exp:performance}
The average return against training iterations in Fig.~\ref{fig:summary} shows that the proposed IC-GAIL and 2IWIL outperform other baselines by a large margin. Due to the mentioned experiment setup, the class prior of the optimal demonstration distribution is around $33\%$. To interpret the experiment results, we would like to emphasize that our experiments are under incomplete optimality setting such that confidence itself is not enough to learn the optimal policy as indicated by the GAIL (Reweight) baseline. Since the difficulty of each task varies, we use different number of $n_c+n_u$ for different tasks. Our contribution is that in addition to the confidence, our methods are able to utilize the demonstration mixture (sub-optimal demonstration) and learn near-optimal policies. 

We can observe that IC-GAIL converges slower than 2IWIL. As discussed in Section~\ref{sec:discussion}, it can be attributed to that the term with respect to the agent in Eq.~\eqref{eq:practical-icgail} is scaled by $0.7$ as specified by $\tau$, which decreases the influence of the agent policy in updating discriminator. The faster convergence of 2IWIL can be an advantage over IC-GAIL when interactions with environments are expensive.
Even though the objective of IC-GAIL becomes biased by not using the class prior $\alpha$, it still converges to near-optimal policies in four tasks.

In Walker2d-v2, the improvement in performance of our methods is not as significant as in other tasks. We conjecture that it is caused by the insufficiency of confidence information. This can be verified by observing that the GAIL (Reweight) baseline in Walker2d-v2 gradually converges to $0.2$ whereas in other tasks it achieves the performance of at least $0.4$. In HalfCheetah-v2, we observe that the discriminator is stuck in a local maximum in the middle of learning, which influences all methods significantly.

The baseline GAIL (Reweight) surpasses GAIL (C) in all tasks, which shows that reweighting enables the agent to learn policies that obtain higher average return. However, since the number of confidence instances is small, the information is not enough to derive the optimal policies. GAIL (U+C) is the standard GAIL without considering confidence information. Although the baseline uses the same number of demonstrations $n_c+n_u$ as our proposed methods, the performance difference is significant due to the use of confidence.

\subsection{Robustness to Gaussian noise in confidence}\label{exp:noise}
\begin{figure}
    \centering
    \includegraphics[scale=0.5]{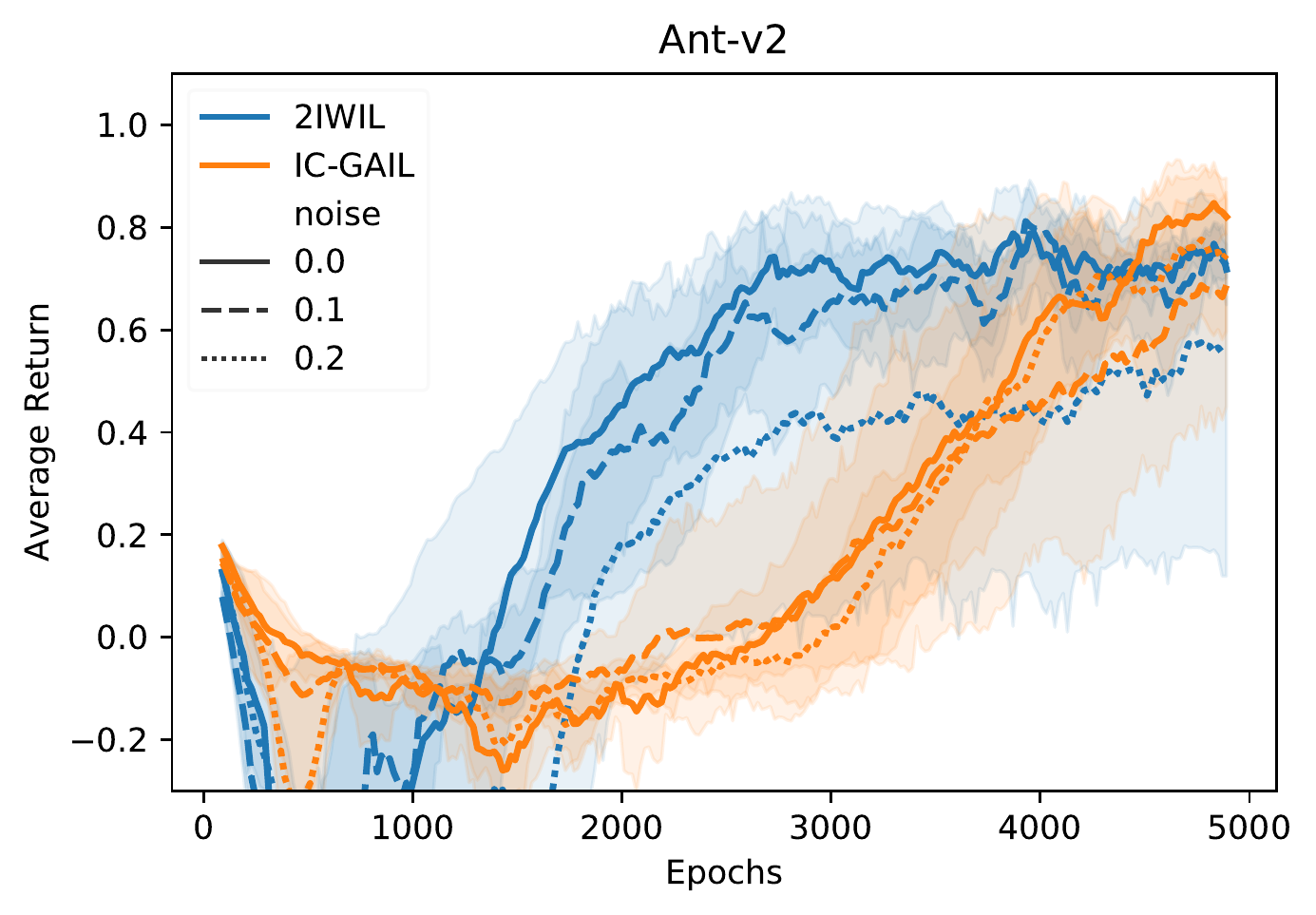}
    \vspace{-0.2in}
    \caption{Learning curves of proposed methods with different standard deviations of Gaussian noise added to confidence. The numbers in the legend indicate the standard deviation of the Gaussian noise.}
    \vspace{-0.1in}
    \label{fig:ant_noise}
\end{figure}
In practice, the oracle that gives confidence scores is basically human labelers and they may not be able to accurately label confidence all the time.
To investigate robustness of our approaches against noise in the confidence scores, we further conduct an experiment on Ant-v2 where the Gaussian noise is added to confidence scores as follows: 
$r(x)=p^\star(y=1\vert x)+\epsilon$, where $\epsilon \sim \mathcal{N}(0, \sigma^2)$. Fig.~\ref{fig:ant_noise} shows the performance of our methods in this noisy confidence scenario. It reveals that both methods are quite robust to noisy confidence, which suggests that the proposed methods are robust enough to human labelers, who may not always correctly assign confidence scores.

\subsection{Influence of unlabeled data}\label{exp:unlabel}
In this experiment, we would like to evaluate the performance of both 2IWIL and IC-GAIL with different numbers of unlabeled data to verify whether unlabeled data is useful. 
As we can see in Fig.~\ref{fig:ant_unlabel}, the performance of both methods grows as the number of unlabeled data increases, which confirms our motivation that using unlabeled data can improve the performance of imitation learning when confidence data is scarce. As discussed in Sec.~\ref{sec:discussion}, the different roles of unlabeled data in the two proposed methods result in dissimilar learning curves with respect to unlabeled data.

\begin{figure}
    \centering
    \includegraphics[scale=0.5]{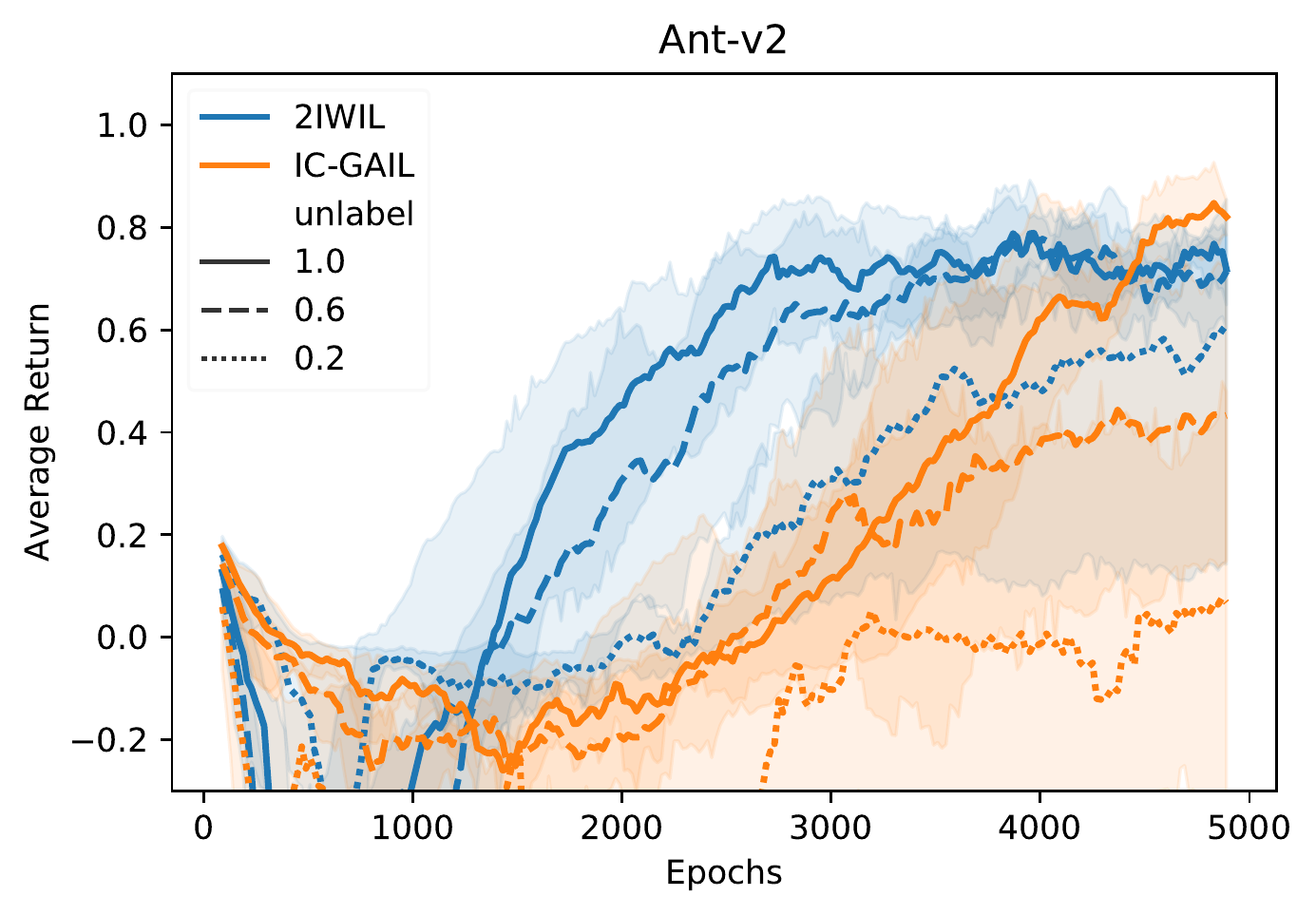}
    \caption{Learning curves of the proposed methods with different number of unlabeled data. The numbers in the legend suggest the proportion of unlabeled data used as demonstrations. $1.0$ is the same as the data used in Fig.~\ref{fig:summary}.}
    \label{fig:ant_unlabel}
\end{figure}
\section{Conclusion}\label{conclusions}
In this work, we proposed two general approaches IC-GAIL and 2IWIL, which allow the agent to utilize both confidence and unlabeled data in imitation learning. 
The setting considered in this paper is usually the case in real-world scenarios because collecting optimal demonstrations is normally costly. In 2IWIL, we utilized unlabeled data to derive a risk estimator and obtained the minimum variance with respect to the combination coefficient $\beta$. 
2IWIL predicts confidence scores for unlabeled data and matches the optimal occupancy measure based on the GAIL objective with importance sampling. For IC-GAIL, we showed that the agent learns an optimal policy by matching a mixture of normalized occupancy measures $p'$ with the normalized occupancy measure of the given demonstrations $p$.

Practically, we conducted extensive experiments to show that our methods outperform baselines by a large margin, to confirm that our methods are robust to noise, and to verify that unlabeled data has a positive correlation with the performance. The proposed approaches are general and can be easily extended to other IL and IRL methods ~\citep{li2017infogail,fu2017learning,kostrikov2018addressing}. 

For future work, we may extend it to a variety of applications such as discrete sequence generation because the confidence in our work can be treated as a property indicator. For instance, to generate soluble chemicals, we may not have enough soluble chemicals, whereas the Crippen function~\citep{crippen19901} can be used to evaluate the solubility as the confidence in this work easily. 

\section*{Acknowledgement}
We thank Zhenghang Cui for helpful discussion. MS was supported by KAKENHI 17H00757, NC was supported by MEXT scholarship, and HB was supported by JST, ACT-I, Grant Number JPMJPR18UI, Japan.

\bibliography{reference}
\bibliographystyle{plainnat}

\newpage
\appendix
\allowdisplaybreaks

\section{Proof for 2IWIL}

\subsection{Proof of Theorem~\ref{lemma:unlabel}}
\begin{theorem*}
The classification risk~\eqref{eq:pnrisk} can be equivalently expressed as 
\begin{align*}
    R_{\mathrm{\mathrm{SC}},\ell}(g)=&\mathbb{E}_{x,r\sim q}[r(\ell(g(x))-\ell(-g(x)))+(1-\beta)\ell(-g(x))]+\mathbb{E}_{x\sim p}[\beta \ell(-g(x))],
\end{align*}
where $\beta\in[0,1]$ is an arbitrary weight.
\end{theorem*}
\begin{proof}
Similar to Eq.~\eqref{eq:bayes}, we may express $p(x\vert y=-1)$ by the Bayes' rule as
\begin{align}\label{eq:bayes_neg}
    p(x\vert y=-1)=\frac{(1-r(x))p(x)}{1-\alpha}.
\end{align}
Consequently, the statement can be confirmed as follows:
\begin{align*}
    R_{\mathrm{\mathrm{SC}},\ell}(g)
    =&\int \alpha p(x\vert y=+1)\ell(g(x))+(1-\alpha)p(x\vert y=-1)\ell(-g(x))dx \nonumber
    \\
    =&\int\alpha\frac{r(x)p(x)}{\alpha}\ell(g(x))+(1-\alpha)\frac{(1 - r(x))p(x)}{1-\alpha}\ell(-g(x))dx
    && \text{($\because$ Eqs.~\eqref{eq:bayes} and~\eqref{eq:bayes_neg})} \nonumber
    \\
    =&\int p(x)r(x)\ell(g(x))+p(x)(1-r(x))\ell(-g(x))dx
    \\
    =& \int\left\{r\ell(g(x)) + (1-r)\ell(-g(x))\right\}q(x, r)dxdr
    \\
    =&\mathbb{E}_{x,r\sim q}[r\ell(g(x))+(1-r)\ell(-g(x))]\\
    =&\mathbb{E}_{x,r\sim q}\left[r\ell(g(x))+(1-r)\ell(-g(x))+ \underbrace{\beta \ell(-g(x))-\beta \ell(-g(x))}_{=0}\right]\\
    =&\mathbb{E}_{x,r\sim q}[r(\ell(g(x))-\ell(-g(x)))+(1-\beta)\ell(-g(x))]+\mathbb{E}_{x\sim p}[\beta \ell(-g(x))]
    .
\end{align*}
\end{proof}

\subsection{Proof of Proposition~\ref{proposition:min_var}}
\begin{proposition*}
Let $\sigma_\mathrm{cov}$ denote the covariance between $n_\C^{-1}\sum_{i=1}^{n_\C}r_i\{\ell(g(x_{\C,i}))-\ell(-g(x_{\C,i}))\}$ and $n_\C^{-1}\sum_{i=1}^{n_\C}\ell(-g(x_{\C,i}))$. For a fixed $g$, the estimator $\hat{R}_{\mathrm{SC},\ell}(g)$ of Eq.~\eqref{eq:weighted}
has the minimum variance when $\beta=\frac{n_\U}{n_\C+n_\U}+\frac{\sigma_\mathrm{cov}}{\mathrm{Var}(\ell(-g(x)))}\frac{n_\C n_\U}{n_\C+n_\U}$ among estimators in the form of Eq.~\eqref{eq:weighted} for $\beta\in[0,1]$.
\end{proposition*}
\begin{proof}
Let 
\begin{align*}
    \mu \triangleq& \mathbb{E}_{\mathcal{D}_\C, \mathcal{D}_\U}[\Hat{R}_{\mathrm{SC},\ell}(g)],\\
    \mu_1\triangleq&\mathbb{E}_{\mathcal{D}_\C}\left[\frac{1}{n_\C}\sum_{i=1}^{n_\C}\ell(-g(x_{\C,i}))\right]=\mathbb{E}_{\mathcal{D}_\U}\left[\frac{1}{n_\U}\sum_{i=1}^{n_\U}\ell(-g(x_{\U,i}))\right] = \E_{x \sim p}[\ell(-g(x))],\\
    w_1\triangleq&\mathbb{E}_{\mathcal{D}_\C}\left[\frac{1}{n_\C}\sum_{i=1}^{n_\C}r(x_i)(\ell(g(x_{\C,i}))-\ell(-g(x_{\C,i})))\right],\\
    w_2\triangleq&\mathbb{E}_{\mathcal{D}_\C}\left[\left(\frac{1}{n_\C}\sum_{i=1}^{n_\C}r(x_i)(\ell(g(x_{\C,i}))-\ell(-g(x_{\C,i})))\right)^2\right],\\
    \lambda \triangleq&\mathbb{E}_{\mathcal{D}_\C}\left[\left(\frac{1}{n_\C}\sum_{i=1}^{n_\C}r(x_i)(\ell(g(x_{\C,i})-\ell(-g(x_{\C,i}))))\right)\left(\frac{1}{n_\C}\sum_{i=1}^{n_\C}\ell(-g(x_{\C,i}))\right)\right],\\
    \sigma_\mathrm{cov}\triangleq&\mathrm{Cov}\left(\frac{1}{n_\C}\sum_{i=1}^{n_\C}r_i(\ell(g(x_{\C,i})-\ell(-g(x_{\C,i})))),\frac{1}{n_\C}\sum_{i=1}^{n_\C}\ell(-g(x_{\C,i}))\right)=\lambda-w_1\mu_1
\end{align*}
We may represent $\mathbb{E}_{\mathcal{D}_\C}\left[\left(\frac{1}{n_\C}\sum_{i=1}^{n_\C}\ell(-g(x_{\C,i}))\right)^2\right]$ in terms of $\mathrm{Var}(\ell(-g(x)))$ and $\mu_1$:
\begin{align*}
    \mathbb{E}_{\mathcal{D}_\C}\left[\left(\frac{1}{n_\C}\sum_{i=1}^{n_\C}\ell(-g(x_{\C,i}))\right)^2\right]
    =&\frac{1}{n_\C^2}\mathbb{E}_{\mathcal{D}_\C}\left[\sum_{i=1}^{n_\C}\ell(-g(x_{\C,i}))^2+2\sum_{i=1}^{n_\C}a\sum_{j=1}^{i-1}\ell(-g(x_{\C,i}))\ell(-g(x_{\C,j}))\right]\\
    =&\frac{1}{n_\C^2}\left(n_\C\mathbb{E}_{x\sim p}\left[\ell(-g(x))^2\right]+n_\C(n_\C-1)\mathbb{E}_{x\sim p}\left[\ell(-g(x))\right]^2\right)\\
    =&\frac{1}{n_\C}\mathrm{Var}(\ell(-g(x)))+\mu_1^2.
\end{align*}
Similarly, we obtain $\E_{\mathcal{D}_\U}[(\frac{1}{n_\U}\sum_{i=1}^{n_\U}\ell(-g(x_{\U,i})))^2] = n_\U^{-1}\mathrm{Var}(\ell(-g(x))) + \mu_1^2$.
As a result,
\begin{align*}
    \mathrm{Var}&(\Hat{R}_{\mathrm{SC},\ell}(g))
    \\
    =&\mathbb{E}_{\mathcal{D}_\C, \mathcal{D}_\U}\left[\left(\Hat{R}_{\mathrm{SC},\ell}(g)\right)^2\right]-\mu^2
    \\
    =& \E_{\mathcal{D}_\C,\mathcal{D}_\U}\left[\underbrace{\frac{1}{n_\C}\sum_{i=1}^{n_\C}r_i(\ell(g(x_{\C,i})) - \ell(-g(x_{\C,i})))}_{\text{(A)}} + (1-\beta)\underbrace{\frac{1}{n_\C}\sum_{i=1}^{n_\C}\ell(-g(x_{\C,i}))}_{\text{(B)}} + \beta\underbrace{\frac{1}{n_\U}\sum_{i=1}^{n_\U}\ell(-g(x_{\U,i}))}_{\text{(C)}}\right] - \mu^2
    \\
    =& \underbrace{w_2}_{\text{(A)}^2} + 2(1-\beta)\underbrace{\lambda}_{\text{(A)(B)}} + 2\beta\underbrace{w_1\mu_1}_{\text{(A)(C)}} + (1-\beta)^2\underbrace{\left(\frac{1}{n_\C}\mathrm{Var}(\ell(-g(x))) + \mu_1^2\right)}_{\text{(B)}^2} \\ &\hspace{140pt} + 2(1-\beta)\beta\underbrace{\mu_1^2}_{\text{(B)(C)}} + \beta^2\underbrace{\left(\frac{1}{n_\U}\mathrm{Var}(\ell(-g(x))) + \mu_1^2\right)}_{\text{(C)}^2} - \mu^2
    \\
    =&\underbrace{\left(w_2+2\lambda-\mu^2+\frac{1}{n_\C}\mathrm{Var}(\ell(-g(x)))+\mu_1^2\right)}_{\mathrm{const. w.r.t. }\beta}
    -2\left(\frac{\mathrm{Var}(\ell(-g(x)))}{n_\C}+\sigma_\mathrm{cov}\right)\beta
    +\mathrm{Var}(\ell(-g(x)))\left(\frac{n_\C+n_\U}{n_\C n_\U}\right)\beta^2
    \\
    =&\mathrm{Var}(\ell(-g(x)))\left(\frac{n_\C+n_\U}{n_\C n_\U}\right)\left(\beta-\left(\frac{n_\U}{n_\C+n_\U}+\frac{\sigma_\mathrm{cov}}{\mathrm{Var}(\ell(-g(x)))}\frac{n_\C n_\U}{n_\C+n_\U}\right)\right)^2
    + \mathrm{const}
    .
\end{align*}
Since $\mathrm{Var}(\ell(-g(x)))\left(\frac{n_\C+n_\U}{n_\C n_\U}\right)\geq 0$, and $\beta\in[0,1]$, $\mathrm{Var}(\Hat{R}_{\mathrm{SC},\ell}(g))$ is minimized when\\ $\beta=\mathrm{clip}_{[0, 1]}\left(\frac{n_\U}{n_\C+n_\U}+\frac{\sigma_\mathrm{cov}}{\mathrm{Var}(\ell(-g(x)))}\frac{n_\C n_\U}{n_\C+n_\U}\right)$. Note that $\mathrm{clip}_{[l, u]}(v)=\min\{\max\{v,l\},u\}$.
\end{proof}

\subsection{Proof of Theorem~\ref{theorem:2iwil-bound}}
\begin{theorem*}
    Let $\mathcal{G}$ be the hypothesis class we use.
    Assume that the loss function $\ell$ is $\rho_\ell$-Lipschitz continuous,
    and that there exists a constant $C_\ell > 0$ such that $\sup_{x \in \mathcal{X}, y \in \{\pm 1\}}|\ell(yg(x))| \leq C_\ell$ for any $g \in \mathcal{G}$.
    Let $\hat{g} \triangleq \argmin_{g \in \mathcal{G}} \hat{R}_{\mathrm{SC},\ell}(g)$ and $g^* \triangleq \argmin_{g \in \mathcal{G}} R_{\mathrm{SC},\ell}(g)$.
    For $\delta \in (0, 1)$, with probability at least $1 - \delta$ over repeated sampling of data for training $\hat{g}$,
    \begin{align*}
        R_{\mathrm{SC},\ell}(\hat{g}) - R_{\mathrm{SC},\ell}(g^*)
        \leq& 16\rho_\ell((3 - \beta)\mathfrak{R}_{n_\C}(\mathcal{G}) + \beta\mathfrak{R}_{n_\U}(\mathcal{G}))
        + 4C_\ell\sqrt{\frac{\log(8/\delta)}{2}}\left((3 - \beta)n_\C^{-\frac{1}{2}} + \beta n_\U^{-\frac{1}{2}}\right).
    \end{align*}
\end{theorem*}

\begin{proof}
    Note that $\hat{g}$ and $g^*$ are the minimizers of $\hat{R}_{\mathrm{SC},\ell}(g)$ and $R_{\mathrm{SC},\ell}(g)$, respectively.
    Then,
    \begin{align*}
        R_{\mathrm{SC},\ell}(\hat{g}) - R_{\mathrm{SC},\ell}(g^*)
        &= R_{\mathrm{SC},\ell}(\hat{g}) - \hat{R}_{\mathrm{SC},\ell}(\hat{g}) + \hat{R}_{\mathrm{SC},\ell}(\hat{g}) - \hat{R}_{\mathrm{SC},\ell}(g^*) + \hat{R}_{\mathrm{SC},\ell}(g^*) - R_{\mathrm{SC},\ell}(g^*)
        \\
        &\leq \sup_{g \in \mathcal{G}}\left(R_{\mathrm{SC},\ell}(g) - \hat{R}_{\mathrm{SC},\ell}(g)\right) + 0 + \sup_{g \in \mathcal{G}}\left(\hat{R}_{\mathrm{SC},\ell}(g) - R_{\mathrm{SC},\ell}(g)\right)
        \\
        &\leq 2\sup_{g \in \mathcal{G}}\left|\hat{R}_{\mathrm{SC},\ell}(g) - R_{\mathrm{SC},\ell}(g)\right|.
    \end{align*}
    From now on, our goal is to bound the uniform deviation $\sup_{g \in \mathcal{G}}\left|\hat{R}_{\mathrm{SC},\ell}(g) - R_{\mathrm{SC},\ell}(g)\right|$.
    Since
    \begin{align}
        \sup_{g \in \mathcal{G}}&\left|\hat{R}_{\mathrm{SC},\ell}(g) - R_{\mathrm{SC},\ell}(g)\right| \nonumber
        \\
        &\leq \sup_{g \in \mathcal{G}}\left|\frac{1}{n_\C}\sum_{i=1}^{n_\C}\left\{r_i(\ell(g(x_{\C,i})) - \ell(-g(x_{\C,i}))) + (1 - \beta)\ell(-g(x_{\C,i}))\right\} \right. \nonumber \\
        & \hspace{100pt} \left. \phantom{\sum} - \E_{x,r\sim q}\left[r(\ell(g(x)) - \ell(-g(x))) + (1 - \beta)\ell(-g(x))\right]\right| \nonumber \\
        & \quad + \beta\sup_{g \in \mathcal{G}}\left|\frac{1}{n_\U}\sum_{i=1}^{n_\U}\ell(-g(x_{\U,i})) - \E_{x \sim p}\left[\ell(-g(x))\right]\right| \nonumber
        \\
        \leq& \sup_{g \in \mathcal{G}}\left|\frac{1}{n_\C}\sum_{i=1}^{n_\C}r_i\ell(g(x_{\C,i})) - \E_{x,r \sim q}[r\ell(g(x))]\right|
        \nonumber+ \sup_{g \in \mathcal{G}}\left|\frac{1}{n_\C}\sum_{i=1}^{n_\C}r_i\ell(-g(x_{\C,i})) - \E_{x,r \sim q}[r\ell(-g(x))]\right| \nonumber \\
        & + (1 - \beta)\sup_{g \in \mathcal{G}}\left|\frac{1}{n_\C}\sum_{i=1}^{n_\C}\ell(-g(x_{\C,i})) - \E_{x,r \sim q}[\ell(-g(x))]\right|
        + \beta\sup_{g \in \mathcal{G}}\left|\frac{1}{n_\U}\sum_{i=1}^{n_\U}\ell(-g(x_{\U,i})) - \E_{x \sim p}[\ell(-g(x))]\right|
        ,
        \label{eq:supp:2iwil-uniform-deviation-bound}
    \end{align}
    all we need to do is to bound four terms appearing in the RHS independently,
    which can be done by McDiarmid's inequality~\citep{mcdiarmid1989method}.
    For the first term, since $\sum_{i=1}^{n_\C}r_i\ell(g(x_{\C,i})) - \E_{x,r \sim q}[r\ell(g(x))]$ is the bounded difference with a constant $C_L/n_\C$ for every replacement of $x_{\C,i}$,
    McDiarmid's inequality state that
    \begin{align*}
        \Pr\Bigg[&\sup_{g \in \mathcal{G}}\left(\frac{1}{n_\C}\sum_{i=1}^{n_\C}r_i\ell(g(x_{\C,i})) - \E_{x,r \sim q}[r\ell(g(x))]\right) \\
        &-\E\left[\sup_{g \in \mathcal{G}}\left(\frac{1}{n_\C}\sum_{i=1}^{n_\C}r_i\ell(g(x_{\C,i}))- \E_{x,r \sim q}[r\ell(g(x))]\right)\right] \geq \varepsilon\Bigg]
        \leq \exp\left(-\frac{2\varepsilon^2}{C_L^2/n_\C}\right),
    \end{align*}
    which is equivalent to
    \begin{align*}
        \sup_{g \in \mathcal{G}}&\left(\frac{1}{n_\C}\sum_{i=1}^{n_\C}r_i\ell(g(x_{\C,i})) - \E_{x,r \sim q}[r\ell(g(x))]\right) \\
        &\qquad\leq \E\left[\sup_{g \in \mathcal{G}}\left(\frac{1}{n_\C}\sum_{i=1}^{n_\C}r_i\ell(g(x_{\C,i})) - \E_{x,r \sim q}[r\ell(g(x))]\right)\right]
        + C_L\sqrt{\frac{\log(8/\delta)}{2n_\C}},
    \end{align*}
    with probability at least $1 - \delta/8$.
    Following the symmetrization device (Lemma 6.3 in~\citet{ledoux1991probability}) and Ledoux-Talagrand's contraction inequality (Theorem 4.12 in~\citet{ledoux1991probability}),
    we obtain
    \begin{align*}
        \E\left[\sup_{g \in \mathcal{G}}\left(\frac{1}{n_\C}\sum_{i=1}^{n_\C}r_i\ell(g(x_{\C,i})) - \E_{x,r \sim q}[r\ell(g(x))]\right)\right]
        &\leq 2\mathfrak{R}_{n_\C}(\ell \circ \mathcal{G})
        && \text{(symmetrization)}
        \\
        &\leq 4\rho_L\mathfrak{R}_{n_\C}(\mathcal{G})
        && \text{(contraction)}
        .
    \end{align*}
    Note that $0 \leq r_i \leq 1$ for $i = 1, \dots, n_\C$.
    Thus, one-sided uniform deviation bound is obtained: with probability at least $1 - \delta/8$,
    \begin{align*}
        \sup_{g \in \mathcal{G}}\left(\frac{1}{n_\C}\sum_{i=1}^{n_\C}r_i\ell(g(x_{\C,i})) - \E_{x,r \sim q}[r\ell(g(x))]\right) \leq 4\rho_L\mathfrak{R}_{n_\C}(\mathcal{G}) + C_L\sqrt{\frac{\log(8/\delta)}{2n_\C}}.
    \end{align*}
    Applying it twice, the two-sided uniform deviation bound is obtained: with probability at least $1 - \delta/4$,
    \begin{align*}
        \sup_{g \in \mathcal{G}}\left|\frac{1}{n_\C}\sum_{i=1}^{n_\C}r_i\ell(g(x_{\C,i})) - \E_{x,r \sim q}[r\ell(g(x))]\right| \leq 8\rho_L\mathfrak{R}_{n_\C}(\mathcal{G}) + 2C_L\sqrt{\frac{\log(8/\delta)}{2n_\C}}.
    \end{align*}
    Similarly, the remaining three terms in the RHS of Eq.~\eqref{eq:supp:2iwil-uniform-deviation-bound} can be bounded.
    Since the second, third, and fourth terms are the bounded differences with constants $C_L/n_\C$, $C_L/n_\C$, and $C_L/n_\U$, respectively,
    the following inequalities hold with probability at least $1 - \delta/4$:
    \begin{align*}
        \sup_{g \in \mathcal{G}}\left|\frac{1}{n_\C}\sum_{i=1}^{n_\C}r_i\ell(-g(x_{\C,i})) - \E_{x,r \sim q}[r\ell(-g(x))]\right| &\leq 8\rho_L\mathfrak{R}_{n_\C}(\mathcal{G}) + 2C_L\sqrt{\frac{\log(8/\delta)}{2n_\C}},
        \\
        \sup_{g \in \mathcal{G}}\left|\frac{1}{n_\C}\sum_{i=1}^{n_\C}\ell(-g(x_{\C,i})) - \E_{x,r \sim q}[\ell(-g(x))]\right| &\leq 8\rho_L\mathfrak{R}_{n_\C}(\mathcal{G}) + 2C_L\sqrt{\frac{\log(8/\delta)}{2n_\C}},
        \\
        \sup_{g \in \mathcal{G}}\left|\frac{1}{n_\U}\sum_{i=1}^{n_\U}\ell(-g(x_{\U,i})) - \E_{x \sim p}[\ell(-g(x))]\right| &\leq 8\rho_L\mathfrak{R}_{n_\U}(\mathcal{G}) + 2C_L\sqrt{\frac{\log(8/\delta)}{2n_\U}}.
    \end{align*}
    After all, we can bound the original estimation error: with probability at least $1 - \delta$,
    \begin{align*}
        R_{\mathrm{SC},\ell}(\hat{g}) - R_{\mathrm{SC},\ell}(g^*)
        \leq& 16\rho_L((3-\beta)\mathfrak{R}_{n_\C}(\mathcal{G}) + \beta\mathfrak{R}_{n_\U}(\mathcal{G})) + 4C_L\sqrt{\frac{\log(8/\delta)}{2}}\left((3 - \beta)n_\C^{-\frac{1}{2}} + \beta n_\U^{-\frac{1}{2}}\right).
    \end{align*}
\end{proof}

\section{Proof for IC-GAIL}


\subsection{Proof of Theorem~\ref{theorem:gan}}
\begin{theorem*}
Denote that
\begin{align*}
    V(\pi_\theta,D_w)=\mathbb{E}_{x\sim p}[\log (1-D_w(x))]+\mathbb{E}_{x\sim p'}[\log D_w(x)],
\end{align*}
and that $C(\pi_\theta)=\max_wV(\pi_\theta,D_w)$.
Then, $V(\pi_\theta, D_w)$ is maximized when $D_w=\frac{p'}{p+p'}(\triangleq D_w^*)$,
and its maximum value is $C(\pi_\theta)=-\log 4+2\mathrm{JSD}(p\|p')$.
Thus, $C(\pi_\theta)$ is minimized if and only if $p_\theta=p_\mathrm{opt}$ almost everywhere.
\end{theorem*}

\begin{proof}
Given a fixed agent policy $\pi_\theta$, the discriminator maximize the quantity $V(\pi_\theta,D_w)$, which can be rewritten in the same way we did in Eq.~\eqref{eq:supp:icgail-equivalence}, such as
\begin{align*}
    V(\pi_\theta, D_w)
    &= \mathbb{E}_{x \sim p}[\log(1 - D_w(x))]+\mathbb{E}_{x \sim p'}[\log D_w(x)]\\
    &=\int p'(x)\log D_w(x) + p(x)\log(1 - D_w(x))dx
    .
\end{align*}
This maximum is achieved when $D_w(x)=D_{w^*}(x)=\frac{p'(x)}{p'(x)+p(x)}$,
with the same discussion as Proposition~1 in \citet{goodfellow2014generative}. As a result, we may derive $\max_wV(\pi_\theta, D_w)$ with $D_w^*(x)$,
\begin{align*}
    C(\pi_\theta)
    =V(\pi_\theta,D_w^*)
    =\mathbb{E}_{x \sim p}\left[\log\frac{p}{p'+p}\right]+\mathbb{E}_{x \sim p'}\left[\log\frac{p'}{p'+p}\right]
    ,
\end{align*}
where $p'=\alpha p_\theta+(1-\alpha)p_\mathrm{non}$.
Note that $C({\pi_\theta})=\mathbb{E}_{x \sim p}[\log\frac{1}{2}]+\mathbb{E}_{x \sim p'}[\log\frac{1}{2}]=-\log4$ when $p'=p$.
We may rewrite $C({\pi_\theta})$ as follows:
\begin{align*}
    C({\pi_\theta})=&\mathbb{E}_{x\sim p}\left[\log\frac{p}{p'+p}\right]+\mathbb{E}_{x\sim p'}\left[\log\frac{p'}{p'+p}\right]\\
    =&-\log 4+\mathbb{E}_{x\sim p}\left[\log\frac{p'}{(p'+p)/2}\right]+\mathbb{E}_{x\sim p'}\left[\log\frac{p}{(p'+p)/2}\right]\\
    =&-\log 4+2\mathrm{JSD}(p\| p'),
\end{align*}
where $\mathrm{JSD}(p_1\|p_2) \triangleq \frac{1}{2}\mathbb{E}_{p_1}[\log\frac{p_1}{(p_1 + p_2) / 2}] + \frac{1}{2}\mathbb{E}_{p_2}[\log\frac{p_2}{(p_1 + p_2) / 2}]$ is Jensen-Shannon divergence.
Since Jensen-Shannon divergence is greater or equal to zero and it is minimized and only if $p'=p$, we obtain that $C({\pi_\theta})$ is minimized if and only if
\begin{align*}
    p'=p
    \Rightarrow \; &\alpha p_\theta+(1-\alpha)p_\mathrm{non}=\alpha  p_\mathrm{opt}+(1-\alpha)p_\mathrm{non} \text{ almost everywhere}\\
    \Rightarrow \; & p_\theta=p_\mathrm{opt}\text{ almost everywhere}.
\end{align*}
\end{proof}

\subsection{Proof of Theorem \ref{theorem:transform}}
\begin{theorem*}
$V(\pi_\theta, D_w)$ can be transformed to $\tilde{V}(\pi_\theta, D_w)$, which is defined as follows:
\begin{align*}
    \tilde{V}(\pi_\theta, D_w) = \mathbb{E}_{x\sim p}[\log (1-D_w(x))]+\alpha\mathbb{E}_{x\sim p_\theta}[\log D_w(x)]+\mathbb{E}_{x,r\sim q}[(1-r)\log D_w(x)].
\end{align*}
\end{theorem*}
\begin{proof}
The statement can be confirmed as follows:
\begin{align}
    \mathbb{E}_{x\sim p}&[\log (1-D_w(x))]+\mathbb{E}_{x\sim p'}[\log D_w(x)] \nonumber
    \\
    &=\mathbb{E}_{x\sim p}[\log (1-D_w(x))]+\alpha\mathbb{E}_{x\sim p_\theta}[\log D_w(x)]+(1-\alpha)\mathbb{E}_{x\sim p_\mathrm{non}}[\log D_w(x)] \nonumber
    \\
    &=\mathbb{E}_{x\sim p}[\log (1-D_w(x))]+\alpha\mathbb{E}_{x\sim p_\theta}[\log D_w(x)]+(1-\alpha)\mathbb{E}_{x,r\sim q}\left[\frac{1-r}{1-\alpha}\log D_w(x)\right] \nonumber
    \\
    &=\mathbb{E}_{x\sim p}[\log (1-D_w(x))]+\alpha\mathbb{E}_{x\sim p_\theta}[\log D_w(x)]+\mathbb{E}_{x,r\sim q}[(1-r)\log D_w(x)]
    ,
    \label{eq:supp:icgail-equivalence}
\end{align}
where the first identity comes from the definition $p' = \alpha p_\theta + (1 - \alpha)p_{\mathrm{non}}$,
and the second identity holds since
\begin{align*}
    \mathbb{E}_{x \sim p_{\mathrm{non}}}[\log D_w(x)]
    &= \int \log D_w(x) p_{\mathrm{non}}(x) dx
    \\
    &= \int \log D_w(x) \frac{1 - r(x)}{1 - \alpha} p(x) dx
    && \text{(note $p_{\mathrm{non}}(x) = p(x \vert y = -1)$)}
    \\
    &= \int \log D_w(x) \frac{1 - r}{1 - \alpha} q(x, r) dx dr
    \\
    &= \mathbb{E}_{x, r \sim q}\left[\frac{1 - r}{1 - \alpha}\log D_w(x)\right]
    .
\end{align*}
\end{proof}

\subsection{Proof of Theorem~\ref{theorem:icgail-bound}}
\begin{theorem*}
    Let $\mathcal{W}$ be a parameter space for training the discriminator and $D_\mathcal{W} \triangleq \{D_w \mid w \in \mathcal{W}\}$ be its hypothesis space.
    Assume that $$\max\{\sup_{x \in \mathcal{X}, w \in \mathcal{W}}|\log D_w(x)|, \sup_{x \in \mathcal{X}, w \in \mathcal{W}}|\log(1-D_w(x))|\} \leq C_L$$,
    and that $\max\{\sup_{w \in \mathcal{W}}|\log D_w(x) - \log D_w(x')|, \sup_{w \in \mathcal{W}}|\log(1-D_w(x)) - \log(1-D_w(x'))|\} \leq \rho_L|x - x'|$ for any $x, x' \in \mathcal{X}$.
    For a fixed agent policy ${\pi_\theta}$, let $D_{\hat{w}} \triangleq \argmax_{w \in \mathcal{W}}\hat{V}({\pi_\theta}, D_w)$ and $D_{w^*} \triangleq \argmax_{w \in \mathcal{W}}V({\pi_\theta}, D_w)$.
    For $\delta \in (0, 1)$, with probability at least $1 - \delta$ over repeated sampling of data for training $D_{\hat{w}}$,
    \begin{align*}
        V({\pi_\theta}, D_{w^*}) - V({\pi_\theta}, D_{\hat{w}})
        \leq& 16\rho_L(\mathfrak{R}_{n_\U}(D_\mathcal{W}) + \alpha\mathfrak{R}_{n_\A}(D_\mathcal{W}) + \mathfrak{R}_{n_\C}(D_\mathcal{W}))
        \\&+ 4C_L\sqrt{\frac{\log(6/\delta)}{2}}\left(n_\U^{-\frac{1}{2}} + \alpha n_\A^{-\frac{1}{2}} + n_\C^{-\frac{1}{2}}\right).
    \end{align*}
\end{theorem*}
\begin{proof}
    Denote $\mathcal{V}(w) \triangleq V({\pi_\theta}, D_w)$ and $\hat{\mathcal{V}}(w) \triangleq \hat{V}({\pi_\theta}, D_w)$.
    Note that $\hat{w}$ and $w^*$ are the minimizers of $\mathcal{V}(w)$ and $\hat{\mathcal{V}}(w)$, respectively.
    Then,
    \begin{align*}
        \mathcal{V}(w^*) - \mathcal{V}(\hat{w})
        &= \mathcal{V}(w^*) - \hat{\mathcal{V}}(w^*)
        + \hat{\mathcal{V}}(w^*) - \hat{\mathcal{V}}(\hat{w})
        + \hat{\mathcal{V}}(\hat{w}) - \mathcal{V}(\hat{w})
        \\
        &\le \sup_{w \in \mathcal{W}}\left(\mathcal{V}(w) - \hat{\mathcal{V}}(w)\right) + 0 + \sup_{w \in \mathcal{W}}\left(\hat{\mathcal{V}}(w) - \mathcal{V}(w)\right)
        \\
        &\le 2\sup_{w \in \mathcal{W}}\left|\hat{\mathcal{V}}(w) - \mathcal{V}(w)\right|.
    \end{align*}
    From now on, our goal is to bound the uniform deviation $\sup_{w \in \mathcal{W}}\left|\hat{\mathcal{V}}(w) - \mathcal{V}(w)\right|$.
    Since
    \begin{align}
        \sup_{w \in \mathcal{W}}\left|\hat{\mathcal{V}}(w) - \mathcal{V}(w)\right|
        &\le \sup_{w \in \mathcal{W}}\left|\frac{1}{n_\U}\sum_{i=1}^{n_\U}\log(1 - D_w(x_{\U,i})) - \mathbb{E}_{x \sim p}\left[\log(1 - D_w(x))\right]\right| \nonumber \\
        &\quad + \alpha \sup_{w \in \mathcal{W}}\left|\frac{1}{n_\A}\sum_{i=1}^{n_\A}\log D_w(x_{\A,i}) - \mathbb{E}_{x \sim p_\theta}\left[\log D_w(x)\right]\right| \nonumber \\
        &\quad +\sup_{w \in \mathcal{W}}\left|\frac{1}{n_\C}\sum_{i=1}^{n_\C}(1-r_i)\log D_w(x_{\C,i}) - \mathbb{E}_{x,r \sim q}\left[(1-r) \log D_w(x)\right]\right|,
        \label{eq:decomposed-uniform-deviation}
    \end{align}
    three terms appearing in the RHS must be bounded independently, utilizing McDiarmid's inequality~\citep{mcdiarmid1989method}.
    For the first term, since $\sum_{i=1}^{n_\U} \log(1 - D_w(x_{\U,i})) - \mathbb{E}_{x \sim p}[\log(1 -D_w(x))]$ has the bounded difference property with a constant $C_L/n_\U$ for every replacement of $x_{\U,i}$,
    we can conclude by McDiarmid's inequality that
    \begin{align*}
        \mathrm{Pr}&\left[\sup_{w \in \mathcal{W}} \left(\sum_{i=1}^{n_\U}\log(1 - D_w(x_{\U,i})) - \mathbb{E}_{x \sim p}[\log(1 - D_w(x))]\right)\right. \\
        &\qquad \left.- \E\left[\sup_{w \in \mathcal{W}} \sum_{i=1}^{n_\U}\log(1 - D_w(x_{\U,i})) - \mathbb{E}_{x \sim p}[\log(1 - D_w(x))]\right] \geq \varepsilon\right] \leq \exp\left(-\frac{2\varepsilon^2}{C_L^2 / n_\U}\right),
    \end{align*}
    which is equivalent to
    \begin{align*}
        \sup_{w \in \mathcal{W}}&\left(\sum_{i=1}^{n_\U}\log(1 - D_w(x_{\U,i})) - \mathbb{E}_{x \sim p}[\log(1 - D_w(x))]\right) \\
        &\leq \mathbb{E}\left[\sup_{w \in \mathcal{W}} \sum_{i=1}^{n_\U}\log(1 - D_w(x_{\U,i})) - \mathbb{E}_{x \sim p}[\log(1 - D_w(x))]\right] + C_L\sqrt{\frac{\log(6/\delta)}{2n_\U}},
    \end{align*}
    with probability at least $1 - \delta/6$.
    Following symmetrization device (Lemma~6.3 in \citet{ledoux1991probability}) and Ledoux-Talagrand's contraction inequality (Theorem~4.12 in \citet{ledoux1991probability}), we obtain
    \begin{align*}
        \mathbb{E}\left[\sup_{w \in \mathcal{W}} \sum_{i=1}^{n_\U}\log(1 - D_w(x_{\U,i})) - \mathbb{E}_{x \sim p}[\log(1 - D_w(x))]\right]
        &\leq 2\mathfrak{R}_{n_\U}(\log\circ D_\mathcal{W})
        && \text{(symmetrization)}
        \\
        &\leq 4\rho_L\mathfrak{R}_{n_\U}(D_\mathcal{W}).
        && \text{(contraction inequality)}
    \end{align*}
    Thus, one-sided uniform deviation bound is obtained: with probability at least $1 - \delta / 6$,
    \begin{align*}
        \sup_{w \in \mathcal{W}}\left(\sum_{i=1}^{n_\U}\log(1 - D_w(x_{\U,i})) - \mathbb{E}_{x \sim p}[\log(1 - D_w(x))]\right)
        \leq 4\rho_L\mathfrak{R}_{n_\U}(D_\mathcal{W}) + C_L\sqrt{\frac{\log(6/\delta)}{2n_\U}}.
    \end{align*}
    Applying it twice, the two-sided uniform deviation bound is obtained: with probability at least $1 - \delta / 3$,
    \begin{align*}
        \sup_{w \in \mathcal{W}}\left|\sum_{i=1}^{n_\U}\log(1 - D_w(x_{\U,i})) - \mathbb{E}_{x \sim p}[\log(1 -  D_w(x))]\right|
        \leq 8\rho_L\mathfrak{R}_{n_\U}(D_\mathcal{W}) + 2C_L\sqrt{\frac{\log(6/\delta)}{2n_\U}}.
    \end{align*}
    Similarly, the second and third terms on the RHS of Eq.~\eqref{eq:decomposed-uniform-deviation} can be bounded.
    Since they have the bounded difference property with constants $C_L/n_\A$ and $C_L/n_\C$, respectively (note that $|1 - r(x)| \leq 1$ for any $x$),
    both of the following inequalities hold independently with probability at least $1 - \delta / 3$:
    \begin{align*}
        \sup_{w \in \mathcal{W}}&\left|\sum_{i=1}^{n_\A}\log D_w(x_{\A,i}) - \mathbb{E}_{x \sim p_\theta}[\log D_w(x)]\right|
        \leq 8\rho_L\mathfrak{R}_{n_\A}(D_\mathcal{W}) + 2C_L\sqrt{\frac{\log(6/\delta)}{2n_\A}},
        \\\sup_{w \in \mathcal{W}}&\left|\sum_{i=1}^{n_\C}(1 - r_i)\log D_w(x_{\C,i}) - \mathbb{E}_{x, r \sim q}[(1 - r)\log D_w(x)]\right|
        \leq 8\rho_L\mathfrak{R}_{n_\C}(D_\mathcal{W})+ 2C_L\sqrt{\frac{\log(6/\delta)}{2n_\C}}.
    \end{align*}
    Combining the above all, we can bound the original estimation error:
    the following bound holds with probability at least $1 - \delta$,
    \begin{align*}
        \mathcal{V}(w^*) - \mathcal{V}(\hat{w})
        \leq& 16\rho_L(\mathfrak{R}_{n_\U}(D_\mathcal{W}) + \alpha\mathfrak{R}_{n_\A}(D_\mathcal{W}) + \mathfrak{R}_{n_\C}(D_\mathcal{W}))
        + 4C_L\sqrt{\frac{\log(6/\delta)}{2}}\left(n_\U^{-\frac{1}{2}} + \alpha n_\A^{-\frac{1}{2}} + n_\C^{-\frac{1}{2}}\right).
    \end{align*}
\end{proof}

\section{Implementation and Experimental Details}
We use the same neural net architecture and hyper-parameters for all tasks. For the architectures of all neural networks, we use two hidden layers with size $100$ and Tanh as activation functions. Please refer to Table~\ref{table:hyper} for more details. Specification of each tasks is shown in Table~\ref{table:tasks}, where we show the average return of the optimal and the uniformly random policies. The average return is used to normalize the performance in Sec.~\ref{experiments} so that $1.0$ indicates the optimal policy and $0.0$ the random policy. 
\begin{table}[!htbp]
\caption{Hyper-parameters used for all tasks.}
\label{table:hyper}
\vskip 0.15in
\begin{center}
\begin{small}
\begin{sc}
\begin{tabular}{llr}
\toprule
Hyper-parameters & value \\
\midrule
$\gamma$                        & $0.995$\\
$\tau$ (Generalized Advantage Estimation)       & $0.97$ \\
batch size                      & $5,000$\\
learning rate (value network)   & $3\times 10^{-4}$ \\
learning rate (discriminator)   & $1\times 10^{-3}$\\
optimizer                       & Adam\\
loss function (2IWIL)           & logistic loss\\
\bottomrule
\end{tabular}
\end{sc}
\end{small}
\end{center}
\end{table}

\begin{table}[!htbp]
\caption{Specification of each tasks. Optimal policy and random policy columns indicate the average return.}
\label{table:tasks}
\vskip 0.15in
\begin{center}
\begin{small}
\begin{sc}
\begin{tabular}{llcllccr}
\toprule
Tasks & $\mathcal{S}$ & $\mathcal{A}$ & $n_u$ & $n_c$ & optimal policy & random policy\\
\midrule
HalfCheetah-v2  & $\mathbb{R}^{17}$ & $\mathbb{R}^6$& 2000  & 500   & 3467.32   & -288.44\\
Walker-v2       & $\mathbb{R}^{17}$ & $\mathbb{R}^6$& 1600  & 400   & 3694.13   & 1.91\\
Ant-v2          & $\mathbb{R}^{111}$& $\mathbb{R}^8$& 480   & 120   & 4143.10   & -72.30\\
Swimmer-v2      & $\mathbb{R}^{8}$  & $\mathbb{R}^2$& 20    & 5     & 348.99    & 2.31\\
Hopper-v2       & $\mathbb{R}^{11}$ & $\mathbb{R}^3$& 16    & 4     & 3250.67   & 18.04\\
\bottomrule
\end{tabular}
\end{sc}
\end{small}
\end{center}
\end{table}


\subsection{Non-negative risk estimator}
By observing the risk estimator of Eq.~\eqref{eq:weighted}, it is possible that the empirical estimation is negative and this may lead to overfitting~\citep{kiryo2017positive}. 
Since we know that the expected risk is nonnegative, we can borrow the idea from \citet{kiryo2017positive} to mitigate this problem by simply adding the max operator to prevent the empirical risk from becoming negative by first rewriting the empirical risk as
\begin{align}\label{eq:nnrisk-beforemax}
    \Hat{R}_{\mathrm{SC},\ell}(g)&=\Hat{R}_C^+(g)+\Hat{R}_{C,U}^-(g),
\end{align}
where 
\begin{align*}
    \Hat{R}_C^+(g) &=\frac{1}{n_\C}\sum_{i=1}^{n_\C}r(x_{\C,i})\ell(g(x_{\C,i})),\\
\end{align*}
and
\begin{align*}
    \Hat{R}_{C,U}^-(g) &=\frac{1}{n_\C}\sum_{i=1}^{n_\C}(1-\beta-r(x_i))\ell(-g(x_{\C,i}))+\frac{1}{n_\U}\sum_{i=1}^{n_\U}\beta \ell(-g(x_{\U,i})).\\
\end{align*}
Note that $R_{C,U}^-\geq 0$ holds for all $g$. However, it is not the case for $\Hat{R}_{C,U}^-(g)$, which is a potential reason to overfit. 
Based on Eq.~\eqref{eq:nnrisk-beforemax}, we achieve the \emph{non-negative risk estimator} that gives the non-negative empirical risk as follows.

\begin{align}
    \Hat{R}_{\mathrm{SC},\ell}(g)=\Hat{R}_C^+(g)+\max\left\{0,\Hat{R}_{C,U}^-(g)\right\}.
\end{align}

\subsection{Ant-v2 Figures}

\begin{figure}[!htbp]
    \centering
    \includegraphics[scale=0.7]{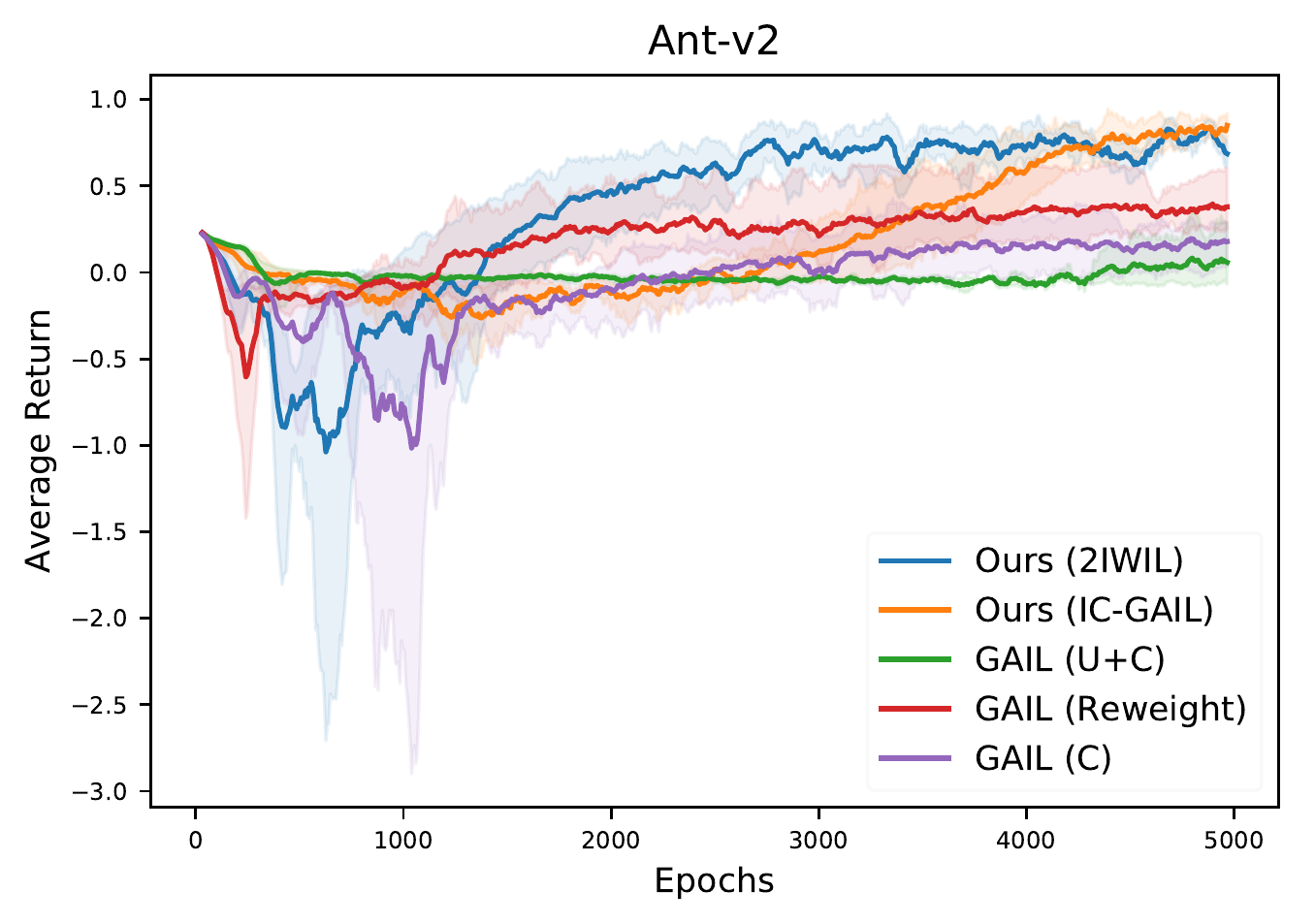}
    \caption{Learning curves of our 2IWIL and IC-GAIL versus baselines.}
    \label{fig:whole_Ant}
\end{figure}
\begin{figure}[!htbp]
    \centering
    \includegraphics
    [scale=0.7]{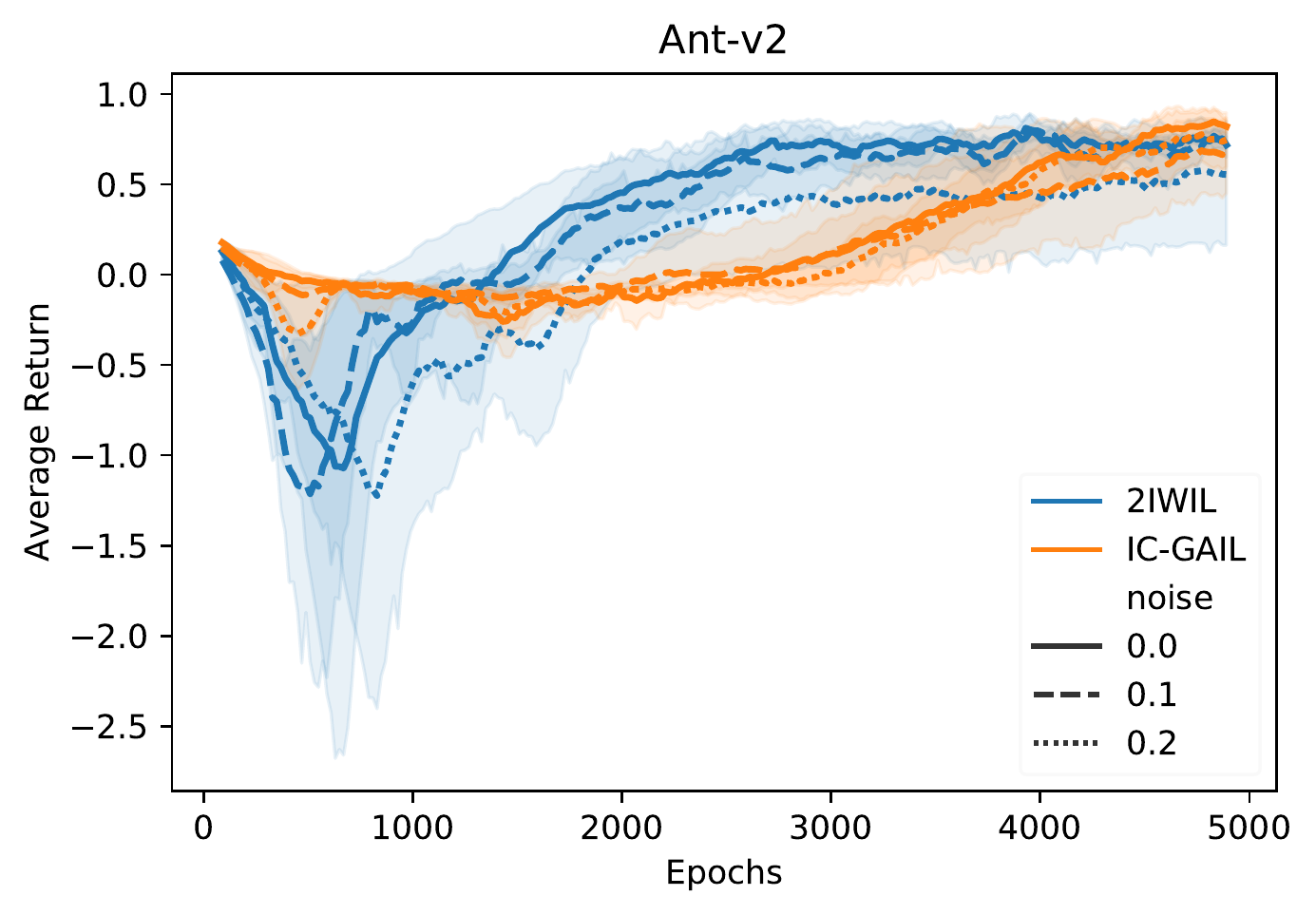}
    \caption{Learning curves of proposed methods with different standard deviations of Gaussian noise added to confidence. The numbers in the legend indicate the standard deviation of the Gaussian noise.}
    \label{fig:whole_noise}
\end{figure}%
\begin{figure}[!htbp]
    \centering
    \includegraphics[scale=0.7]{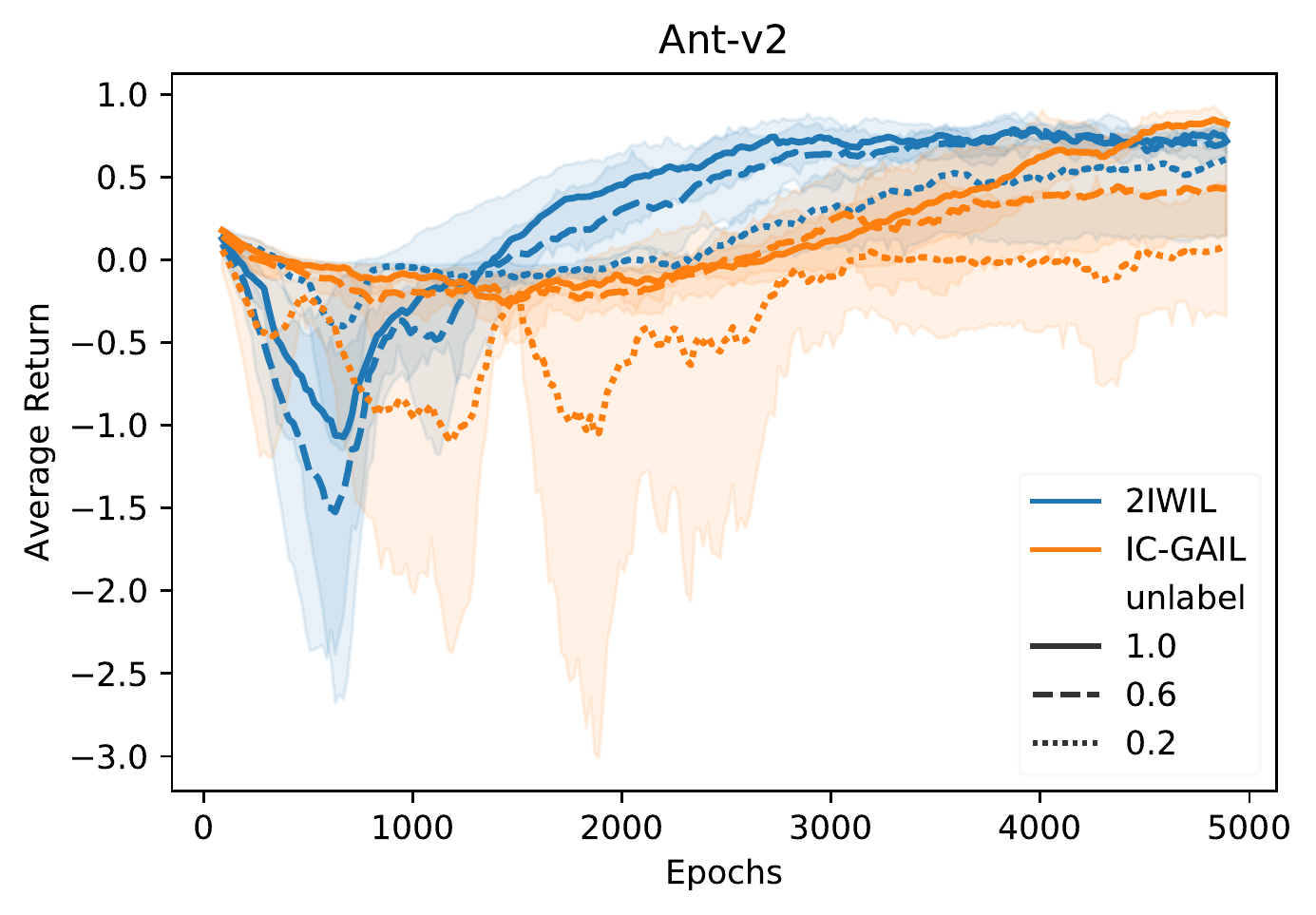}
    \caption{Learning curves of the proposed methods with different number of unlabeled data. The numbers in the legend suggest the proportion of unlabeled data used as demonstrations.}
    \label{fig:whole_unlabel}
\end{figure}

We empirically found that when using GAIL-based approaches in Ant-v2 environment, the performance degrades quickly in early training stages. The uncropped figures are Figs.~\ref{fig:whole_Ant}, \ref{fig:whole_noise} and~\ref{fig:whole_unlabel}.

\end{document}